\documentclass[a4paper,11pt]{article}

\usepackage[utf8]{inputenc}
\usepackage[inline]{enumitem}
\usepackage{amsmath,amssymb,amsthm}
\usepackage{mathtools}
\mathtoolsset{showonlyrefs}
\usepackage[margin=0.05\linewidth,font=small,labelfont=bf]{caption}
\DeclareCaptionType{alg}[Algorithm][List of Algorithms]
\usepackage[colorlinks=True]{hyperref}
\usepackage[linesnumbered,lined]{algorithm2e}
\usepackage{tikz}
\usepackage{pgfplots}
\usepgfplotslibrary{groupplots}
\pgfplotsset{compat=1.17}
 
\newcommand\eg{e.\,g.}
\newcommand\ie{i.\,e.}
\newcommand\st{s.\,t.}
\newcommand\wrt{w.\,r.\,t.}

\newtheorem{asm}{Assumption}
\newtheorem{dfn}{Definition}
\newtheorem{lem}{Lemma}
\newtheorem{thm}{Theorem}
\newtheorem{cor}{Corollary}

\newcommand\sz[1]{\lvert#1\rvert}
\newcommand\norm[1]{\lVert#1\rVert}
\newcommand\set[1]{\ensuremath{\mathcal{#1}}}
\renewcommand\vec[1]{\ensuremath{\boldsymbol{#1}}}

\newcommand\RR{\mathbb R}

\newcommand\T{\mathsf T}

\newcommand\sX{\set X}
\newcommand\sY{\set Y}
\newcommand\sH{\set H}
\newcommand\sD{\set D}
\newcommand\sP{\set P}
\newcommand\sS{\set S}
\newcommand\sT{\set T}
\newcommand\sF{\set F}
\newcommand\sG{\set G}
\newcommand\sW{\set W}

\newcommand\vx{\vec x}
\newcommand\vy{\vec y}
\newcommand\vA{\vec A}
\newcommand\vb{\vec b}
\newcommand\vu{\vec u}
\newcommand\vw{\vec w}

\renewcommand\leq{\leqslant}
\renewcommand\geq{\geqslant}

\newcommand\loss{\ensuremath{\mathrm{Loss}}}

\def\(#1\){\begin{align}#1\end{align}}

\title{Hierarchical Partitioning Forecaster}
\author{
Christopher Mattern\\
\texttt{cmattern@deepmind.com}
}
\date{12 May 2023}

\begin{document}

\maketitle

\begin{abstract}
In this work we consider a new family of algorithms for sequential prediction, Hierarchical Partitioning Forecasters (HPFs).
Our goal is to provide appealing theoretical - regret guarantees on a powerful model class - and practical - empirical performance comparable to deep networks - properties \emph{at the same time}.
We built upon three principles:
hierarchically partitioning the feature space into sub-spaces, blending forecasters specialized to each sub-space and learning HPFs via local online learning applied to these individual forecasters.
Following these principles allows us to obtain regret guarantees, where Constant Partitioning Forecasters (CPFs) serve as competitor.
A CPF partitions the feature space into sub-spaces and predicts with a fixed forecaster per sub-space.
Fixing a hierarchical partition \sH\ and considering any CPF with a partition that can be constructed using elements of \sH\ we provide two guarantees:
first, a generic one that unveils how local online learning determines regret of learning the entire HPF online;
second, a concrete instance that considers HPF with linear forecasters (LHPF) and exp-concave losses where we obtain $O(k \log T)$ regret for sequences of length $T$ where $k$ is a measure of complexity for the competing CPF.
Finally, we provide experiments that compare LHPF to various baselines, including state of the art deep learning models, in precipitation nowcasting.
Our results indicate that LHPF is competitive in various settings.
\end{abstract}

\section{Introduction}

\paragraph{Background.}
Sequential prediction is an important problem in machine learning.
State of the art machine learning methods are driven by experimental analysis, by improving existing network architectures, introducing paradigms on network design and all that taking the advantage of the massively increasing memory and compute resources to push the limits of model capacity.
From an empirical point of view the machine learning approach is extremely successful, however, this comes at a downside - there is no guarantee if and how good performance carries forwards.
This holds true especially in the light of the IID assumptions at the heart of the typical machine learning training procedures and often causes trouble when we face data that is out of distribution relative to training and validation data.
This leads to a disconnect between experimental and theoretic research:
Increasing model capacity plays a major role in experimental research, yet a minor role in theoretical research (highly complex models in a real-world setting hardly allow for a theoretic analysis);
handling out of sample data plays a significant role in theoretical research (simple models in a well-defined setting allow for a theoretical analysis), yet a not so significant one in experimental research.
In this work we want to bridge the gap between experimental and theoretic research by considering a non-trivial model that can provide performance comparable to high capacity models and at the same time satisfies theoretic guarantees. 

The family of Gated Linear Networks (GLNs) \cite{veness2020gln, veness2020ggln} obeys similar characteristics and is similar in spirit.
GLNs are composed of layers of gated neurons.
The input to each neuron is considered a set of predictions and a neuron outputs a combined and (hopefully) more accurate prediction.
This process carries forward across layers.
Combining predictions involves gating based on side-information and allows for specializing neurons to that side-information;
learning does not rely on back-propagation, rather each neuron learns the global optimization goal online and in isolation using Online Gradient Descent \cite{zinkevich2003ogd}.
GLNs provide good empirical performance in various settings and enjoy theoretic guarantees on model capacity.
However, there is a downside:
there are no regret guarantees for GLNs that hold for individual sequences.

In this work we focus on sequential prediction and adopt local online learning and specialization to side-information, yet exploit them differently compared to GLNs.
(Note that conceptually we do not need to distinguish between features and side-information, rather we view side-information as a part of features.)
For specialization we partition the feature space into sub-spaces and repeat this recursively, ultimately this induces a \emph{hierarchical partition} of the feature space.
Every sub-space, \emph{segment}, of this hierarchical partition has a specialized forecaster.
To compute a forecast we recursively combine the predictions of the specialized forecasters while we work through the hierarchical partition starting from the highest degree of specialization.
This procedures gives rise to the family of Hierarchical Partitioning Forecasters (HPFs).
We also provide an online local-learning procedure for learning HPFs that guarantees low regret \wrt\ an idealized predictor that uses a feature space partitioning based on the involved hierarchical partitioning with fixed predictors, both chosen optimal in hindsight, for individual sequences.
So unlike GLNs our approach allows for guarantees for individual sequnences.
We found that the key to provide these regret guarantees is to rely on the hierarchical structure.
(As of now for GLNs, where gating does not exploit any structure across neurons, and, is arbitrary in this sense, there was no fruitful attempt to obtain such regret guarantees.)
This reasoning has its roots in Context Tree Weighting and its successors \cite{willems1995ctw,willems1998ctw,veness2012cts}.

\paragraph{Outline and Contribution.}
In the remainder of this work we first introduce some general notation and definitions in Section~\ref{sec:notation}.
We then present our main contributions:

First, in Section~\ref{sec:meta-hpf}, we propose a meta-algorithm for learning HPFs that employs learning algorithms (learners) for individual forecasters specialized to the segments of a hierarchical partition and we provide a regret analysis.
Overall the setting is very generic:
we neither provide concrete learners, nor we provide which (class of) forecasters we consider, rather we think of these as templates that come with a regret guarantee.
Our regret analysis then links the meta-algorithm's total regret to the regret of learning the structure of an arbitrary competing partition and to the regret of learning the forecasters specialized to segments of that partition.

Second, in Section~\ref{sec:lhpf}, we furthermore consider a concrete instance of the meta-algorithm that learns HPFs with linear functions as forecasters - Linear HPF (LHPF) given exp-concave loss functions (\eg\ log loss, MSE, see \cite{hazan2006ons} for more examples).
Here learners are based on second order online optimization \cite{hazan2006ons} and a generalization of Switching \cite{erven2007switching}, an ensembling technique related to Fixed Share \cite{herbster1995fixedshare}.

Third, in Section~\ref{sec:experiments}, we provide a short experimental study on a topic that became popular among the deep learning community, short-term precipitation forecasting (nowcasting).
We consider precipitation nowcasting in the UK as previously done in \cite{ravuri2021gan}.
Our results suggest that our learning algorithm for LHPFs provides results comparable to significantly more complex deep-learning models in various settings, yet we also highlight limitations.

We finally summarize our results and outline topics for future research in Section~\ref{sec:conclusion}.

\section{Basic Notation and Definitions}
\label{sec:notation}

\paragraph{General Notation.}
To ease reading of symbols we adopt the following typesetting conventions:
non-boldface symbols denote scalars ($x, X$), boldface lowercase symbols ($\vx$) denote column vectors, boldface uppercase symbols ($\vec X$) denote matrices, calligraphic symbols denote sets ($\set X$).
Some column vector $\vec x$ has components $(x_1, x_2, \dots)^\T$.
Let $\angle(\vec a, \vec b)$ denote the angle between vectors $\vec a$ and $\vec b$.
Let $\nabla_{\vec x} := (\frac\partial{\partial x_1}, \frac{\partial}{\partial x_2}, \dots)^\T$ denote the gradient operator \wrt\ \vx\ and let $\log := \log_e$ denote the natural logarithm.
We use $x_{1:t}$ to denote a sequence $x_1x_2 \dots x_t$ of objects (if $t=\infty$ the sequence has infinite length), define the shorthand $x_{<t} := x_{1:t-1}$ and say $x_{1:t}$ is a sequence over \sX, if $x_1, x_2, \dots\in\sX$.
For objects $x_a, x_b, \dots$ with labels $a, b, \dots$ from some set $\sS$ let $\{x_s\}_{s\in\sS}$ denote an indexed multiset of objects.

\paragraph{Segments and (Hierarchical) Partitions.}
A \emph{partition} $\sP$ of a non-empty set $\sX$ is a set of disjoint non-empty sets \st\ their union is $\sX$; an element of $\sP$ is called \emph{segment}.

\begin{dfn}\label{dfn:hierarchical-partition}
\sH\ is a \emph{hierarchical partition} of a non-empty set \sX\, if
\begin{enumerate}[label=(\roman*)]
    \item
    $\sH = \{\sX\}$ or
    \item
    $\sH = \sH_1 \cup \dots \cup \sH_n \cup \{\sX\}$, where $n\geq2$, $\{\sX_1, \sX_2, \dots, \sX_n\}$ is a partition of $\sX$ and $\sH_1, \sH_2, \dots, \sH_n$ are hierarchical partitions of $\sX_1, \sX_2, \dots, \sX_n$.
\end{enumerate}
We say 
\begin{itemize}
    \item
    \emph{partition \sP\ (of \sX) is induced by \sH}, if $\sP\subseteq\sH$;
    \item
    \emph{segment $\sS'$ divides segment \sS\ (``\sS\ is divisible'')}, for $\sS', \sS\in\sH$, if $\sS'\subset\sS$ and no segment $\sT\in\sH$ exists \st\ $\sS' \subset\sT\subset\sS$; 
    \item 
    \emph{segment $\sS\in\sH$ is indivisible}, if no segment from \sH\ divides \sS.
\end{itemize} 
\end{dfn}

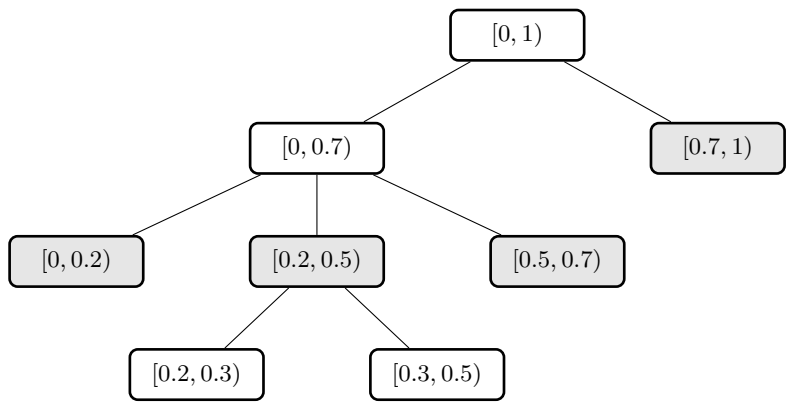
\begin{figure}
    \centering
    \begin{tikzpicture}[
        every node/.style={
            draw,
            rectangle,
            rounded corners=3pt,
            line width=1pt,
            minimum width=50pt,
            minimum height=15pt,
            inner sep=5pt,
            align=center,
            font=\footnotesize},
        level 1/.style={sibling distance=150pt},
        level 2/.style={sibling distance=90pt},
        ]
        \node () {$[0, 1)$}
            child{ node {$[0, 0.7)$}
                child{ node[fill=gray!20] {$[0, 0.2)$} }
                child{ node[fill=gray!20] {$[0.2, 0.5)$}
                    child{ node {$[0.2, 0.3)$} }
                    child{ node {$[0.3, 0.5)$} }
                }
                child{ node[fill=gray!20] {$[0.5, 0.7)$} }
            }
            child{ node[fill=gray!20] {$[0.7, 1)$} }
        ;
    \end{tikzpicture}
    \caption{%
    Tree-analogy for hierarchical partition $\{[0, 1),\allowbreak [0, 0.7),\allowbreak [0.7, 1.0),\allowbreak [0, 0.2),\allowbreak [0.2, 0.5),\allowbreak [0.5, 0.7),\allowbreak [0.2, 0.3),\allowbreak [0.3, 0.5)\}$ of $[0, 1)$;
    nodes are labeled with corresponding segments.
    }
    \label{fig:hierarchical-partition}
\end{figure}

\noindent
Note that the concept of a hierarchical partition of a set resembles that of a tree where nodes are labelled with sets.
It is useful to keep this equivalence in mind throughout this work as it makes later proofs very intuitive.
Let us illustrate the equivalence (examples refer to Figure~\ref{fig:hierarchical-partition}):
A segment $\sS$ corresponds to a node and all segments dividing $\sS$ correspond to the children of that node, \eg\ all segments dividing $[0, 0.7)$ are $[0, 0.2)$, $[0.2, 0.5)$, $[0.5, 0.7)$, representing child nodes.
Consequently, if segment $\sS$ corresponds to node $u$ and segment $\sS'$ corresponds to node $u'$, then $\sS'\subset\sS$ translates to $u'$ is a descendant of $u$.
For instance $[0.2, 0.3) \subset [0, 0.7)$, hence represents a descendant.
Furthermore, divisible segments corresponds to internal nodes and indivisible segments corresponds to leaf nodes, \eg\ $[0, 1)$, $[0, 0.7)$ and $[0.2, 0.5)$ represent all internal nodes and the remaining nodes represent all leaf nodes.

We now further give examples for the terminology from Definition~\ref{dfn:hierarchical-partition}.
Segments $\{[0, 0.2), [0.2, 0.5), [0.5, 0.7), [0.7, 1)\}$ (highlighted in gray) form a partition of $[0, 1)$ which is a subset of the hierarchical partition from the above example, hence it is an induced partition.
We have $[0.2, 0.3)\subset[0.2, 0.5)\subset[0, 0.7)$, hence segment $[0.2, 0.3)$ doesn't divide $[0, 0.7)$, but $[0.2, 0.5)$ does.
There is no segment that divides $[0.2, 0.3)$, hence this segment is indivisible.

\paragraph{Sequential Forecasting, CPF and HPF.}
We consider \emph{sequential prediction for individual sequences}.
In this setting a \emph{forecaster} operates in rounds:
In a round $t$ we first observe \emph{features} $\vx_t$ from \emph{feature space} $\sX$ and then a forecaster construct a forecast $\vy_t$ from \emph{forecast space} $\sY$.
We represent the forecaster as a mapping $f:\sX\rightarrow\sY$, so $\vy_t = f(\vx_t)$.
Finally, a loss function $\ell:\sY\rightarrow\RR$ is revealed and the forecaster suffers loss $\ell(\vy_t)$, ending the round.
Typically forecasters and losses can vary with $t$.

We now formally define two forecasters that will play a key role in the following.

\begin{dfn}\label{dfn:cpf}
A \emph{Constant Partitioning Forecaster (CPF)} $f$ is given by a partition \sP\ of feature space \sX\ and by forecasters $\{f^{\sS}\}_{\sS\in\sP}$.
It forecasts $f(\vx) := f^\sS(\vx)$, for the unique segment $\sS\in\sP$ with $\vx\in\sS$.
\end{dfn}
\noindent
Observe that a CPF $f$ suffers total loss
\(
    \loss^f_T = \sum_{\sS\in\sP} \sum_{\substack{t:\vx_t\in\sS,\\1\leq t\leq T}} \ell_t(f^\sS(\vx_t))
    \text.
    \label{eq:cpf-loss}
\)

\begin{dfn}\label{dfn:hpf}
A \emph{Hierarchical Partitioning Forecaster (HPF)} $f$ is given by a hierarchical partition \sH\ of feature space \sX,
forecasters $\{f^{\sS}\}_{\sS\in\sD}$ with signature $f^{\sS}:\sX\times\sY\rightarrow\sY$, for the subset $\sD$ of divisible segments from \sH,
forecasters $\{g^{\sS}\}_{\sS\in\sH\setminus\sD}$ with signature $g^{\sS}:\sX\rightarrow\sY$ for the subset $\sH\setminus\sD$ of indivisible segments from \sH.
It forecasts $f(\vx) := h^{\sX}(\vx)$, where we recursively define
\(
    h^{\sS}(\vx) :=
    \begin{cases}
        g^{\sS}(\vx)\text, & \text{if $\vx\in\sS$ and $\sS$ is indivisible,}\\
        f^{\sS}(\vx, h^{\sS'}(\vx))\text, & \text{if $\vx\in\sS'$ and $\sS'$ divides $\sS$,}\\
        \text{undefined}, & \text{if $\vx\notin\sS$.}
    \end{cases}
    \label{eq:hpf-recursion}
\)
\end{dfn}

\paragraph{Sequantial Learning.}
Sequential learning is a natural mechanism to choose forecasters.
We now formalize this approach for our purposes, since we will later rely on it.

\begin{dfn}\label{dfn:sequential-learner}
A \emph{Sequential Learner} $L: s, f, \ell, \vx \mapsto s', f'$ maps a state $s$, a forecaster $f$, a loss function $\ell$ and a feature vector \vx\ to forecaster $f'$ and state $s'$.
For a sequence of loss functions $\ell_1, \ell_2, \dots$ and feature vectors $\vx_1, \vx_2, \dots$, initial forecaster $f_1$ and initial state $s_1$ the sequential learner generates a trajectory of forecasters $f_2, f_3, \dots$ defined by $(s_{t+1}, f_{t+1}) := L(s_t, f_t, \ell_t, \vx_t)$, for $t\geq1$.
We say $L$ has parameters $s_1$ and $f_1$.
\end{dfn}
\noindent
A sequential learner $L$ attempts to perform close to a desirable forecaster $f$ in the sense that the excess total loss - \emph{regret} - of the generated forecasters $f_1, f_2, \dots$ over $f$ grows sublinearily,
\(
    \sum_{\mathclap{1\leq t\leq T}}\mkern4mu(c_t(f_t) - c_t(f)) = o(T)\text{, where $c_t(f) := \ell_t(f(\vx_t))$}
    \text.
    \label{eq:learner-cost}
\)
To ease later proofs we say \emph{$f$ has regret at most $R(T, f_1, f)$ under sequential learner $L$}, if for the forecasters $f_1, f_2, \dots$ generated by $L$ we have
\(
    \sum_{1\leq t\leq T}(c_t(f_t) - c_t(f)) \leq R(T, f_1, f)
    \text.
    \label{eq:regret-bound}
\)

\section{Meta-Algorithm}
\label{sec:meta-hpf}

\subsection{Learning HPFs Sequentially}
Algorithm~\ref{fig:hpf} sequentially generates a sequence of HPFs.
In round $t$ the HPF parameters are updated by applying sequential learning to forecasters $f^\sS$ and $g^\sS$ that were involved in computing the current rounds forecast $f_t(\vx_t)$ (\ie\ for all $\sS$ that contain $\vx_t$, cf. \eqref{eq:hpf-recursion}).
All other forecasters remain unchanged.
Learning is \emph{local}, there is no backpropagation.
As we will see in the next section local learning enables a regret analysis.

\begin{alg}
\centering
\small
\fboxsep=0.4\baselineskip
\fbox{\begin{minipage}{0.9\textwidth}
    \noindent
    \begin{tabular}{@{}l@{}@{\hskip2pt}c@{\hskip2pt}p{0.8\columnwidth}@{}}
        \textbf{Input}&\textbf:&
        Hierarchical partition $\sH$ with divisible segments $\sD$,\newline
        sequential learner $L_0$ with parameters $\{(u^{\sS}_1, f^{\sS}_1)\}_{\sS\in\sD}$,\newline
        sequential learner $L_1$ with parameters $\{(v^{\sS}_1, g^{\sS}_1)\}_{\sS\in\sH\setminus\sD}$,\newline
        features $\vx_{1:T}$ and
        loss functions $\ell_{1:T}$.
        \\[0.2\baselineskip]
        \textbf{Output}&\textbf:& Predictions $\vy_{1:T}$.
        \\[0.2\baselineskip]
    \end{tabular}
    \\[.5\baselineskip]
    For $t = 1, 2, \dots, T$ do:
    \begin{enumerate}[label=\arabic*.,ref=\arabic*]
        \item
        Consider HPF $f_t$ with parameters $(\sH, \{f_t^\sS\}_{\sS\in\sD}, \{g_t^\sS\}_{\sS\in\sH\setminus\sD})$.
        \item
        Observe $\vx_t$ and output prediction $\vy_t = f_t(\vx_t)$.
        \item
        Observe $\ell_t$ and update learner parameters, for all $\sS\in\sH$:
        \begin{enumerate}[label=\textit{Case~\theenumi\alph*:},ref=\theenumi\alph*,wide]        
            \item
            $\vx_t\in\sS$ and $\sS$ is indivisible.
            \(
                (v^{\sS}_{t+1}, g^\sS_{t+1})
                = L_1(v^{\sS}, g_t^\sS, \ell_t, \vx_t)
                \text.
            \)
            \item
            $\vx_t\in\sS'$ where $\sS'\in\sH$ divides $\sS$.
            \(
                (u^{\sS}_{t+1}, f^\sS_{t+1})
                = L_0(u^{\sS}_t, f_t^\sS, \ell_t, \vx')
                \text{, where }
                \vx' = (\vx_t, h^{\sS'}_t(\vx_t))
                \text.
            \)
            (For $h_t$ see \eqref{eq:hpf-recursion} setting $f^{\sS} =f_t^{\sS}$ and $g^{\sS} = g^{\sS}_t$.)
            \item
            $\vx_t\notin\sS$. Retain states and forecasters, $u_{t+1}=u_t$, $f_{t+1}^\sS = f_t^\sS$, similarly for $v^{\sS}_t$ and $g^{\sS}_t$.
        \end{enumerate}
    \end{enumerate}
\end{minipage}}
\caption{Learning a sequence of HPFs.}
\label{fig:hpf}
\end{alg}

\subsection{Analysis}

\paragraph{Overview.}
We now proceed with the regret analysis of Algorithm~\ref{fig:hpf}, taking CPFs as competitors.
For this we first impose some technical constraints.
Next, we investigate the virtue of local online learning on the forecasting functions $f^\sS$.
Finally, based on, this we show that the regret can be split into two components:
regret by learning the structure of a CPF and regret by learning the forecasting functions of a CPF.

\paragraph{Technical Constraints.}
In the remaining part of this section we assume:

\begin{asm}
\label{asm:hpf}
Fix regret bounds $R_0$ and $R_1$ and let $\sG$ be the set of all forecasters with regret bound $R_1$ under $L_1$.
We assume:
\begin{enumerate}[label=(\roman*), ref=\ref{asm:hpf}(\roman*)]
    \item\label{asm:hpf-sanity}
    The set $\sG$ is non-empty.
    \item\label{asm:hpf-embedding}
    For any $g\in\sG$ there exists $f$ with regret at most $R_0$ under $L_0$ \st\ $f(\vx, \vy) = g(\vx)$, for all $(\vx, \vy)\in\sX\times\sY$.
    For any fixed $g$ the set of all such $f$ has a minimizer.\footnote{%
    Technically this is not required, yet it greatly eases the upcoming proofs.
    This similarly holds for the next assumption.
    }
    \item\label{asm:hpf-projection}
    There exists $f$ with regret at most $R_0$ under $L_0$ \st\ $f(\vx, \vy) = \vy$, for all $(\vx, \vy)\in\sX\times\sY$.
    The set of all such $f$ has a minimizer.
\end{enumerate}
\end{asm}

\noindent
Let us now briefly discuss the purpose of the above assumptions.
In general we should think of $R_0$ and $R_1$ as sufficiently good regret guarantees and $\sG$ as a set of desirable forecasters.
First, Assumption~\ref{asm:hpf-sanity} serves as a regularity condition to avoid pathological edge cases.
Assumption~\ref{asm:hpf-embedding} ensures that set $\sG$ is embedded in the set of forecasters that is learnable with regret bound $R_0$ under $L_0$.
Assumption~\ref{asm:hpf-projection} furthermore ensures that forecasters that just forward another prediction are learnable with regret bound $R_0$ under $L_0$.
In the next section we will explore the effects of these properties.

\paragraph{Virtue of Local Sequential Learning.}
We turn towards the loss accumulated at an arbitrary segment in the course of Algorithm~\ref{fig:hpf} which works out to
\(
    \loss^\sS_T := \sum_{t:\vx_t\in\sS}\ell_t(h^\sS_t(\vx_t))
    \text{,\enskip for $h^\sS_t$ see Algorithm~\ref{fig:hpf} and \eqref{eq:hpf-recursion}.}
    \label{eq:loss-hpf-partition}
\)

\begin{lem}[Local Sequential Learning]\label{lem:loss-hpf-partition}
Let Assumption~\ref{asm:hpf} hold and consider some segment $\sS\in\sH$ in the course of Algorithm~\ref{fig:hpf} and let $n := \lvert\{t:\vec x_t\in\sS\}\rvert$.
We have:
\begin{enumerate}[label=(\roman*),align=right,leftmargin=*,ref=\ref{lem:loss-hpf-partition}(\roman*)]
    \item\label{lem:loss-hpf-partition-a}
    If $\sS$ is indivisible, then for any $g\in\sG$ we have
    \(
        \loss^\sS_T
        \leq \sum_{\mathclap{t: \vx_t\in\sS}} \ell_t(g(\vx_t)) + p^\sS_T(g)
        \text{,\enskip for\enskip }
        p^\sS_T(g) := R_1(n, g^\sS_1, g)
        \text.
    \)
    \item\label{lem:loss-hpf-partition-b}
    If $\sS$ is divisible, then for any $g\in\sG$ we have
    \(
        \loss^\sS_T
        \leq \sum_{\mathclap{t: \vx_t\in\sS}} \ell_t(g(\vx_t)) + q^\sS_T(g)
        \text{,\enskip for\enskip }
        q^\sS_T(g) := \min_{f\in\dots} R_0(n, f^\sS_1, f)
        \text,
        \label{eq:loss-hpf-partition-b-regert}
    \)
    where the minimum is over all $f$ that satisfy Assumption~\ref{asm:hpf-embedding} for $g$.
    \item\label{lem:loss-hpf-partition-c}
    If $\sS$ is divisible, then 
    \(
        \loss^\sS_T
        \leq \sum_{\mathclap{\sS'\text{\,div.\,}\sS}} \loss^{\sS'}_T + r^\sS_T
        \text{,\enskip for\enskip }
        r^\sS_T := \min_{f\in\dots} R_0(n, f^\sS_1, f)
        \text,
        \label{eq:loss-hpf-partition-c-regert}
    \)
    where the minimum is over all $f$ that satisfy Assumption~\ref{asm:hpf-projection}.    
\end{enumerate}
\end{lem}
\begin{proof}
%Let $h_i = g^\sS_t$ where $t$ is the $i$-th time step \st\ $\vx_t\in\sS$, similarly to the definition of $c_i$.
For brevity let $h_i = h^\sS_t$ and $c_i(f) := \ell_t(f(\vx_t))$, where $t$ is the $i$-th time step \st\ $\vx_t\in\sS$.
Based on this we obtain
\(
    \loss^\sS_T = \sum_{1\leq i\leq n} c_i(h_i)
    \label{eq:loss-gln-partition-proof-0}
\)
and distinguish:

\medskip\par\noindent\textit{Case~1: \sS\ is indivisible.} ---
Sequential learner $L_1$ generates sequence $h_1, h_2, \dots$ of forecasters (see Algorithm~\ref{fig:hpf}), where $h_1 = g^\sS_1$.
For any $g\in\sG$ we have
\(
    \sum_{1\leq i\leq n} c_i(h_i)
    &\stackrel{\text{\ref{it:loss-hpf-partition-proof-0}}}\leq
    \sum_{1\leq i\leq n} c_i(g) + R_1(n, h_1, g)
    \\
    &\stackrel{\text{\ref{it:loss-hpf-partition-proof-1}}}=
    \sum_{t:\vx_t\in\sS} \ell_t(g(\vx_t)) + R_1(n, g_1^\sS, g)
    \text,
\)
where we used
\begin{enumerate*}[label=(\alph*)]
    \item\label{it:loss-hpf-partition-proof-0}
    $g$ has regret bound $R_1$ under $L_1$ and
    \item\label{it:loss-hpf-partition-proof-1}
    the definition of the $c_i$'s and $h_1 = f^\sS_1$.
\end{enumerate*}
This proves Lemma~\ref{lem:loss-hpf-partition-a}.

\medskip\par\noindent\textit{Case~2: \sS\ is divisible.} ---
Sequential learner $L_0$ generates sequence $h_1, h_2, \dots$, of forecasters (see Algorithm~\ref{fig:hpf}), where $h_1 = f^\sS_1$.
Similarly to Case~1, for any $f$ with regret bound $R_0$ under $L_0$ we have
\(
    \sum_{1\leq i\leq n} c_i(h_i)
    &\leq
    \sum_{1\leq i\leq n} c_i(f) + R_0(n, f^\sS_1, f)
    \text.
    \label{eq:loss-hpf-partition-proof-2}
\)

\medskip\par\noindent\textit{Lemma~\ref{lem:loss-hpf-partition-b}:}
By Assumption~\ref{asm:hpf-embedding} we can choose $f$ that satisfies $f(\vx, \vy) = g(\vx)$ and at the same time minimizes $R_0(n, f^{\sS}_1, \cdot)$ for the desired $g$, so
\(
    \sum_{1\leq i\leq n} c(f) = \sum_{t:\vx\in\sS} \ell_t(g(\vx_t))
    \text.
\)
Combining this with \eqref{eq:loss-hpf-partition-proof-2} yields Lemma~\ref{lem:loss-hpf-partition-b}.

\medskip\par\noindent\textit{Lemma~\ref{lem:loss-hpf-partition-c}:}
We get
\(
    \sum_{1\leq i\leq n} c_i(f)
    &\stackrel{\text{\ref{it:loss-hpf-partition-proof-3}}}=
    \sum_{\text{$\sS'$\,div.\,\sS}}\sum_{t:\vx_t\in\sS'} \ell_t(f(\vx_t, h^{\sS'}_t(\vx_t)))
    \\
    &\stackrel{\text{\ref{it:loss-hpf-partition-proof-4}}}=
    \sum_{\text{$\sS'$\,div.\,\sS}}\sum_{t:\vx_t\in\sS'} \ell_t(g^{\sS'}_t(\vx_t))
    \\
    &\stackrel{\text{\ref{it:loss-hpf-partition-proof-5}}}=
    \sum_{\text{$\sS'$\,div.\,\sS}} \loss_T^{\sS'}
    \text,
    \label{eq:loss-hpf-partition-proof-3}
\)
where we used
\begin{enumerate*}[label=(\alph*)]
    \item\label{it:loss-hpf-partition-proof-3}
    the definition of the $c_i$'s and \eqref{eq:hpf-recursion},
    \item\label{it:loss-hpf-partition-proof-4}
    by Assumption~\ref{asm:hpf-projection} we can choose $f$ \st\ $f(\vx, \vy) = \vy$ and that it minimizes the $R_1$-term in \eqref{eq:loss-hpf-partition-proof-2} at the same time and
    \item\label{it:loss-hpf-partition-proof-5}
    Equation~\eqref{eq:loss-hpf-partition}.
\end{enumerate*}
We plug \eqref{eq:loss-hpf-partition-proof-3} into \eqref{eq:loss-hpf-partition-proof-2} and by our choice of $f$ we conclude the proof.
\end{proof}

\paragraph{Structure and Total Loss.}
From Algorithm~\ref{fig:hpf} we see that the total loss works out to
\(
    \loss_T^{\mathrm{HPF}}
    = \sum_{1\leq t\leq T} \ell_t(h^\sX_t(\vx_t))
    \text.
\)
By applying Lemma~\ref{lem:loss-hpf-partition-c} recursively we can associate the loss $\loss_T^{\mathrm{HPF}}$ to the local loss of a set of arbitrary segments from \sH\ that form a partition of \sX.
This guarantees that HPF can compete with CPFs with any partition induced by \sH.

\begin{lem}[Structure Loss]\label{lem:hpf-structure-regret}
Let Assumption~\ref{asm:hpf} hold and consider Algorithm~\ref{fig:hpf}.
For any partition $\sP$ induced by hierarchical partition $\sH$ of $\sX$ we have
\(
    \loss^{\mathrm{HPF}}_T
    \leq
    \sum_{\sS\in\sP} \loss^\sS_T
    + \sum_{\mathclap{\substack{\sS\in\sH: \\\sS\supset\sS'\in\sP}}}r^\sS_T
    \text.
    \label{eq:gln-structure-loss-0}
\)
For the term $r^\sS_T$ see Lemma~\ref{lem:loss-hpf-partition}.
\end{lem}
\begin{proof}
Our prove is by induction on $\sz\sP$.

\medskip\par\noindent\textit{Base: $\sz\sP=1$.} ---
We get
\(
    \loss^{\mathrm{HPF}}_T
    \stackrel{\text{\ref{it:hpf-structure-loss-proof-0}}}=
    \loss^\sX_T
    \stackrel{\text{\ref{it:hpf-structure-loss-proof-1}}}=
    \sum_{\sS\in\sP} \loss^\sS_T + \sum_{\mathclap{\substack{\sS\in\sH: \\\sS\supset\sS'\in\sP\\}}}r^\sS_T
    \text,
\)
where we used
\begin{enumerate*}[label=(\alph*)]
    \item\label{it:hpf-structure-loss-proof-0}
    Definition~\ref{dfn:hpf} and \eqref{eq:loss-hpf-partition} and
    \item\label{it:hpf-structure-loss-proof-1} 
    the sum over $r^\sS_T$ is empty, since
    $\sP=\{\sX\}$ (the sole partition of $\sX$ of size one) which implies $\{\sS\in\sH :\allowbreak \sS\supset\sS'\allowbreak\text{ for }\sS'\in\sP \} = \emptyset$.
\end{enumerate*}

\medskip\par\noindent\textit{Step: $\sz\sP>1.$} ---
There exist pairwise disjoint segments $\sS_1, \sS_2, \dots, \sS_n\in\sP$, where $n\geq2$, and $\sS'\in\sH$ such that $\sS_1, \sS_2, \dots, \sS_n$ divide $\sS'$ and $\sS_1 \cup \sS_2\cup\dots\cup\sS_n = \sS'$.
Hence, the set $\sP' := \sP \setminus\{\sS_1, \sS_2, \dots, \sS_n\}\cup\sS'$ also is a partition induced by $\sH$ and has strictly lower cardinality than $\sP$.
We conclude
\(
    \loss^{\mathrm{HPF}}_T
    &\stackrel{\text{\ref{it:hpf-structure-loss-proof-2}}}\leq
    \sum_{\sS\in\sP'\setminus\{\sS'\}} \loss^\sS_T
    + \loss^{\sS'}_T
    + \sum_{\mathclap{\substack{\sS\in\sH: \\\sS\supset\sS'\in\sP'}}}r^\sS_T
    \\
    &\stackrel{\text{\ref{it:hpf-structure-loss-proof-3}}}\leq
    \sum_{\sS\in\sP'\setminus\{\sS'\}} \loss^\sS_T
    + \sum_{\mathclap{1\leq i\leq n}} \loss^{\sS_i}_T
    + r^{\sS'}_T
    + \sum_{\mathclap{\substack{\sS\in\sH: \\\sS\supset\sS'\in\sP'}}}r^\sS_T
    \\
    &\stackrel{\text{\ref{it:hpf-structure-loss-proof-4}}}=
    \sum_{\substack{\sS\in\sP\\\phantom{\sS\in\sP':}}} \loss^\sS_T + \sum_{\mathclap{\substack{\sS\in\sH: \\\sS\supset\sS'\in\sP\\}}}r^\sS_T
    \text,
\)
where we used
\begin{enumerate*}[label=(\alph*)]
    \item\label{it:hpf-structure-loss-proof-2}
    the induction hypothesis for $\sP'$,
    \item\label{it:hpf-structure-loss-proof-3} 
    Lemma~\ref{lem:loss-hpf-partition-c} for $\loss^{\sS'}_T$ and
    \item\label{it:hpf-structure-loss-proof-4} 
    $\sP = \sP'\setminus\{\sS'\}\cup\{\sS_1, \dots, \sS_2\}$, $\{\sS\in\sH:\allowbreak\sS\subset\sS'\text{ for }\sS'\in\sP\} \allowbreak = \{\sS'\}\cup\allowbreak\{\sS\in\sH:\sS\subset\sS'\text{ for }\sS'\in\sP'\}$ and $\sS_1, \dots, \sS_n$ partition $\sS'$, \ie\ there is no multiple summation (all by construction).
    \qedhere
\end{enumerate*}
\end{proof}

\noindent
To obtain our first main result it remains to argue that the sequential learners can learn a desirable forecaster at every segment of a HPF.
(Lemma~\ref{lem:loss-hpf-partition-a} for indivisible segments and Lemma~\ref{lem:loss-hpf-partition-b} for divisible segments).

\begin{thm}[Total Loss]\label{thm:hpf-total-loss}
Let Assumption~\ref{asm:hpf} hold and consider Algorithm~\ref{fig:hpf}.
For any CPF with partition $\sP$ induced by hierarchical partition $\sH$ and forecasters $\{f^\sS\}_{\sS\in\sP}\subseteq\sG$ we have
\(
    \loss^{\mathrm{HPF}}_T
    \leq \loss^{\mathrm{CPF}}_T
    + \sum_{\mathclap{\substack{\sS\in\sP: \\\text{$\sS$ indiv.}}}}p^\sS_T(f^\sS)
    + \sum_{\mathclap{\substack{\sS\in\sP: \\\text{$\sS$ div.}}}}q^\sS_T(f^\sS)
    + \sum_{\mathclap{\substack{\sS\in\sH: \\\sS\supset\sS'\in\sP\\}}}r^\sS_T
    \text.
    \label{eq:hpf-total-loss-0}
\)
For the terms $p^\sS_T$, $q^\sS_T$ and $r^\sS_T$ see Lemma~\ref{lem:loss-hpf-partition}.
\end{thm}
\begin{proof}
We get
\(
    \loss^{\mathrm{HPF}}_T
    &\stackrel{\text{\ref{it:hpf-total-loss-proof-0}}}\leq
    \sum_{{\substack{\sS\in\sP:\\\text{\sS\ indiv.}}}}\loss^\sS_T
    + \sum_{{\substack{\sS\in\sP:\\\text{\sS\ div.}}}}\loss^\sS_T
    + \sum_{\mathclap{\substack{\sS\in\sH: \\\sS\supset\sS'\in\sP\\}}}r^\sS_T
    \\
    &\stackrel{\text{\ref{it:hpf-total-loss-proof-1}}}\leq
    \sum_{{\substack{\sS\in\sP:\\\text{\sS\ indiv.}}}}\left[\sum_{t:\vx_t\in\sS}\ell_t(f^\sS(\vx_t))+ p^\sS_T(f^\sS)\right]
    \\
    &\phantom\leq\qquad+
    \sum_{{\substack{\sS\in\sP:\\\text{\sS\ div.}}}}\left[\sum_{t:\vx_t\in\sS}\ell_t(f^\sS(\vx_t)) + q^\sS_T(f^\sS)\right]
    + \sum_{\mathclap{\substack{\sS\in\sH: \\\sS\supset\sS'\in\sP\\}}}r^\sS_T
    \\
    &=
    \sum_{\sS\in\sP}\sum_{t:\vx_t\in\sS}\ell_t(f^\sS(\vx_t))
    + \sum_{\mathclap{\substack{\sS\in\sP: \\\text{$\sS$ indiv.}}}}p^\sS_T(f^\sS)
    + \sum_{\mathclap{\substack{\sS\in\sP: \\\text{$\sS$ div.}}}}q^\sS_T(f^\sS)
    + \sum_{\mathclap{\substack{\sS\in\sH: \\\sS\supset\sS'\in\sP\\}}}r^\sS_T
\)
and finally \eqref{eq:hpf-total-loss-0} follows from Definition~\ref{dfn:cpf} and \eqref{eq:cpf-loss}.
In the above we
\begin{enumerate*}[label=(\alph*)]
    \item\label{it:hpf-total-loss-proof-0}
    applied Lemma~\ref{lem:hpf-structure-regret} and split \sP\ into indivisible and divisible segments and
    \item\label{it:hpf-total-loss-proof-1}
    applied Lemma~\ref{lem:loss-hpf-partition-a} to indivisible segments and Lemma~\ref{lem:loss-hpf-partition-b} to divisible segments.
\end{enumerate*}
\end{proof}

%%% START %%%

\section{Learning Linear HPFs (LHPFs)}
\label{sec:lhpf}

\paragraph{Overview.}
In this section we consider an important special case of the sequential prediction problem, that is forecasting a scalar (forecast space $\sY\subseteq\RR)$ given $n$-dimensional feature vectors (feature space $\sX\subseteq\RR^n$).
To tackle this problem we propose populating HPF with linear functions over some parameter space $\sW\subseteq\RR^n$ as forecasters for individual segments.
In the following we fill in the gaps in Algorithm~\ref{fig:hpf}:
First, we introduce learners for divisible and indivisible segments that imply $O(\log T)$ regret \wrt\ a CPF with linear functions with parameters from \sW\ as forecasters.
Second, we explain how to choose the hierarchical partition \sH.

\paragraph{Exp-Concavity.}
The upcoming analysis and loss bounds are based on sufficient curvature of the underlying loss functions $\ell_1, \ell_2, \dots$.
The curvature we demand is slightly stronger than just convexity, that is:

\begin{dfn}
For $\eta>0$ some function $f:\sX\rightarrow\RR$, where $\sX\subseteq\RR^n$, is \emph{$\eta$-exp concave}, if $e^{-\eta f}$ is concave.
\end{dfn}

\noindent
Observe that for an $\eta$-exp concave function $f:\RR\rightarrow\RR$ with scalar domain the function $g:\RR^n\rightarrow\RR$, where $g(\vec x) := f(\vec a^\T\vx)$, for some fixed $\vec a\in\RR^n$, with vector domain trivially also is $\eta$-exp concave.

\subsection{Sequential Learner for Indivisible Segments}
\label{sec:learner-for-indivisible-segments}

\begin{alg}
\centering
\small
\fboxsep=0.4\baselineskip
\fbox{\begin{minipage}{0.9\textwidth}
    \noindent
    \begin{tabular}{@{}l@{}@{\hskip2pt}c@{\hskip2pt}p{0.8\columnwidth}@{}}
        \textbf{Input}&\textbf:&
        Features $\vx_{1:T}$ from $\RR^n$, loss functions $\ell_{1:T}$, parameter $\gamma>0$, parameter space $\mathcal W\subseteq\RR^n$.
        \\[0.2\baselineskip]
        \textbf{Output}&\textbf:& Predictions $y_{1:T}$.
        \\[0.2\baselineskip]
    \end{tabular}
    \\[.5\baselineskip]
    Set initial state $\vA_0 = \vec0, \vb_0=\vec0$ and $\vw_0=\frac1m\cdot\vec1$.\\[0.5\baselineskip]
    For $t = 1, 2, \dots, T$ do:
    \begin{enumerate}[label=\arabic*.,ref=\arabic*]
        \item
        Observe features $\vx_t$ and output prediction $y_t = \vw_{t-1}^\T \vx_t$.
        \item
        Observe $\ell_t$, let $\nabla_t = \nabla_{\vec w}\mkern2mu\ell_t(\vec w^\T\vx_t)\big\rvert_{\vw=\vw_{t-1}}$ and update state
        \(
            \vA_t
            &= \vA_{t-1} + \nabla_t\nabla_t^\T
            \text,\\
            \vb_t
            &= \vb_{t-1} + \left(\nabla_t^\T\vw_{t-1}-\frac1\gamma\right)\cdot \nabla_t
            \text{ and }
            \\
            \vw_t
            &= \arg\min_{\vw\in\mathcal W} \frac12\vw^\T\vA_{t-1}\vw - \vb_{t-1}^\T\vw
            \text.
        \)
    \end{enumerate}
\end{minipage}}
\caption{Learner for linear forecasters (indivisible segments).}
\label{fig:indivisible-segment-learner}
\end{alg}

We now consider efficient sequential learning for the family of linear forecasters with constrained parameters, that is
\(
    \sG = \{f:\RR^n \rightarrow \RR \mid f(\vx) = \vw^\T\vx\text{ and }\vw\in\sW\}
    \text{, for some }
    \sW\subseteq\RR^n
    \text.
    \label{eq:indivisible-segment-forecaster}
\)
Based on previous work on second-order sequential learning Algorithm~\ref{fig:indivisible-segment-learner} specifies a desirable learner, as any $g\in\sG$ has at most logarithmic regret under this learner, given reasonable regularity constraints, as we will see shortly.

Let us take a closer look at Algorithm~\ref{fig:indivisible-segment-learner}.
The algorithm maintains a state $(\vA_t, \vb_t, \vw_t)$ that represents a second-order approximation of the total loss $\sum_t \ell_t(\vx_t\cdot\vw)$ as a function of forecaster parameters \vw\ and the current minimizer $\vw_t$ of that approximate loss.
In round $t$ it chooses the forecaster $\vx\mapsto \vec w_{t-1}\vx$ from \sG\ that minimizes the approximate total loss $\frac12\vw^\T\vA_{t-1}\vw-\vb_{t-1}^\T\vw$ up until round $t-1$.
If the parameter space \sW\ is a compact convex set then minimizing approximate total loss is a convex (\ie\ well-behaved and efficiently solvable) optimization problem, since $\vA_{t-1}$ by construction is positive semi-definite.
This procedure comes with the following regret guarantee.

\begin{lem}[FTAL, see Theorem~6 in \cite{hazan2006ons}]
\label{lem:indivisible-segment-regret}
Consider Algorithm~\ref{fig:indivisible-segment-learner}.
If
\begin{itemize}
    \item
    $\sW\subseteq\RR^n$ is bounded \st\ $\norm{\vec v-\vec u} \leq D$, for all $\vec u, \vec v\in\sW$,
    \item
    $\ell_t$ is $\eta$-exp-concave for all $1\leq t\leq T$,
    % (NOTE: this implies $c_t(\vw) := \ell_t(\vx_t^\T\vw)$ is $\eta$-exp-concave.)
    \item
    $\norm{\nabla \ell_t(\vx_t^\T\vw)}\leq G$, for all $\vw\in\sW$ and all $1\leq t\leq T$ and
    \item
    we choose $\gamma=\frac12\min\left\{\frac1{4GD}, \eta\right\}$
\end{itemize}    
then any $g\in\sG$ has regret at most
\(
    R(T, g_1, g) = A\cdot\left(1 + \log T\right)
    \text{, where }
    A:=64n\left(\frac1\eta+GD\right)
\)
and $g_1(\vx)=\frac1m\vec1^\T\vx$, under Algorithm~\ref{fig:indivisible-segment-learner}.
\end{lem}

\noindent
For the proof we defer the reader to the corresponding reference.

\subsection{Sequential Learner for Divisible Segments}

\begin{alg}
\centering
\small
\fboxsep=0.4\baselineskip
\fbox{\begin{minipage}{0.9\textwidth}
    \noindent
    \begin{tabular}{@{}l@{}@{\hskip2pt}c@{\hskip2pt}p{0.8\columnwidth}@{}}
        \textbf{Input}&\textbf:&
        Features $\vx_{1:T}$ from $\RR^n$,
        predictions $v_{1:T}$ from $\RR$,
        loss functions $\ell_{1:T}$,
        parameter $\gamma>0$,
        parameter space $\mathcal W\subseteq\RR^n$.
        \\[0.2\baselineskip]
        \textbf{Output}&\textbf:& Predictions $y_{1:T}$.
        \\[0.2\baselineskip]
    \end{tabular}
    \\[.5\baselineskip]
    Set initial state $\vA_0 = \vec0, \vb_0=\vec0, \vw_0=\frac1m\cdot\vec1$ and $\vec\beta_0=\left(\frac12~ \frac12\right)^\T$.\\[0.5\baselineskip]
    For $t = 1, 2, \dots, T$ do:
    \begin{enumerate}[label=\arabic*.,ref=\arabic*]
        \item
        Observe features $\vx_t$ and compute base prediction $u_t = \vw_{t-1}^\T\vx_t$.
        \item
        Observe expert prediction $v_t$ and mix with base prediction
        \(
            y_t
            %= \frac{\beta_{t-1}^1\cdot u_t + \beta_{t-1}^2\cdot v_t}{\beta^1_{t-1} + \beta^2_{t-1}}
            = \frac{\vec\beta_{t-1}}{\vec1^\T\vec\beta_{t-1}}\cdot\left(\begin{matrix}u_t\\v_t\end{matrix}\right)
            \text.
            \label{eq:divisible-segment-learner-prediction}
        \)
        \item
        Observe $\ell_t$, let $\nabla_t = \nabla_{\vec w}\mkern2mu\ell_t(\vec w^\T\vx_t)\big\rvert_{\vw=\vw_{t-1}}$, and update state
        \(
            \vA_t
            &= \vA_{t-1} + \nabla_t\nabla_t^\T
            \text,\\
            \vb_t
            &= \vb_{t-1} + \left(\nabla_t^\T\vw_{t-1}-\frac1\gamma\right)\cdot \nabla_t
            \text,\\
            \vw_t
            &= \arg\min_{\vw\in\mathcal W} \frac12\vw^\T\vA_{t-1}\vw - \vb_{t-1}^\T\vw
            \text,\\
            \begin{split}
            \alpha_t
            &=\frac1{t+1}
            \text{ and }\\
            \vec\beta_t
            % &= \left(\begin{matrix}1-\alpha_t&\alpha_t\\\alpha_t&1-\alpha_t\end{matrix}\right)\cdot\left(\begin{matrix}e^{-\eta\ell_t(u_t)}&0\\0&e^{-\eta\ell_t(v_t)}\end{matrix}\right)\vec\beta_{t-1}
            &= \left(\begin{matrix}(1-\alpha_t)e^{-\eta\ell_t(u_t)}&\alpha_t e^{-\eta\ell_t(v_t)}\\\alpha_t e^{-\eta\ell_t(v_t)}&(1-\alpha_t)e^{-\eta\ell_t(v_t)}\end{matrix}\right)\cdot\vec\beta_{t-1}
            \text.
            \end{split}
            \label{eq:divisible-segment-learner-beta-update}
        \)
    \end{enumerate}
\end{minipage}}
\caption{Learner for divisible segments.}
\label{fig:divisible-segment-learner}
\end{alg}

For divisible segments we structure the family of forecasters slightly different compared to \eqref{eq:indivisible-segment-forecaster}, that is
\(
    \sF
    := \{f:\RR^{n+1}\rightarrow\RR \mid f(\vx, y) &= (1-v)\cdot \vw^\T \vx + v \cdot y,
    \\
    v&\in[0, 1]\text{ and }\vw\in\sW\}
    % \sF := \left\{f:\RR^{n+1}\rightarrow\RR\mid f(\vx, y) = \vw\cdot \left(\begin{matrix}\vx\\y\end{matrix}\right)\text{ and }\vw\in\RR^n\times[0, 1]\right\}
    \text.
    \label{eq:divisible-segment-forecaster}
\)
Obviously the class \sG\ of forecasters is embedded in this set.
In Algorithm~\ref{fig:divisible-segment-learner} we depict a learner that, on the one hand, satisfies the technical requirements that we require to state a regret guarantee on Linear HPF and, on the other hand, guarantees low regret.
That is, any $f\in\sF$ has at most logarithmic regret under this learner.

\begin{lem}
\label{lem:divisible-segment-regret}
Consider Algorithm~\ref{fig:divisible-segment-learner} and define $A := 64n\left(\frac1\eta + GD\right)$, $B := \frac1\eta$ and $f_1(\vx, y) = \frac1{2m}\vec1^\T\vx + \frac y2$.
Under the regularity conditions and the choice of $\gamma$ specified in Lemma~\ref{lem:indivisible-segment-regret} we have:
\begin{enumerate}[label=(\roman*)]
    \item
    \label{lem:divisible-segment-regret-a}
    For every $g\in\sG$ there exists $f\in\sF$ \st\ we have $f(\vx, y) = g(\vx)$ for all $\vx, y$.
    Any such $f$ has regret at most
    \(
        R(T, f_1, f) = (A + B)\cdot(1 +\log T)
    \)
    under Algorithm~\ref{fig:divisible-segment-learner}.
    \item
    \label{lem:divisible-segment-regret-b}
    There exists $f\in\sF$ \st\ we have $f(\vx, y) = y$ for all $\vx, y$.
    Any such $f$ has regret at most
    \(
        R(T, f_1, f) = B\cdot (1 + \log T)
    \)
    under Algorithm~\ref{fig:divisible-segment-learner}.
\end{enumerate}
\end{lem}
\begin{proof}
For the proof we first argue on the structure Algorithm~\ref{fig:divisible-segment-learner} and then conclude either statement.

\medskip\par\noindent\textit{Structure: Switching.}
First, note that the prediction \eqref{eq:divisible-segment-learner-prediction} in conjunction with the weight update \eqref{eq:divisible-segment-learner-beta-update} is an instance of Switching with (two) experts predicting $(u_t, v_t)$, switching rate $\alpha_t=(t+1)^{-1}$  and $\eta$-exp-concave loss functions $\ell_t$, all for $1\leq t\leq T$.
From Corollary~\ref{cor:switching-two-expert-regret-bound} we obtain
\(
    \sum_{1\leq t\leq T} \ell_t(y_t)
    \leq 
    \sum_{1\leq t\leq T} \ell_t(z_t) + B\cdot (1 + \log T)
    \text{,\enskip for\enskip } z_t\in\{u_t, v_t\}
    \text,
    \label{eq:divisible-segment-regret-proof-0}
\)
where we considered switching sequence $i_1=i_2=\dots=1$ for $z_t=u_t$ (no switch), similarly for $z_t=v_t$.

\medskip\par\noindent\textit{Structure: Algorithm~\ref{fig:indivisible-segment-learner}.}
Second, note that Algorithm~\ref{fig:indivisible-segment-learner} (with regularity conditions and parameter choice as in Lemma~\ref{lem:indivisible-segment-regret}) is embedded in  Algorithm~\ref{fig:divisible-segment-learner} to compute the base prediction $u_t$.
So Lemma~\ref{lem:indivisible-segment-regret} implies 
\(
    \sum_{1\leq t\leq T} \ell_t(u_t)
    \leq
    \sum_{1\leq t\leq T} \ell_t(g(\vx_t)) + A \cdot (1 + \log T)
    \text{,\enskip for any $g\in\sG$.}
    \label{eq:divisible-segment-regret-proof-1}
\)

\medskip\par\noindent\textit{Lemma~\ref{lem:divisible-segment-regret}\ref{lem:divisible-segment-regret-a}.}
Fix any $g\in\sG$, where $g(\vx) = \vw^\T\vx$ and $\vw\in\sW$, we get
\(
    \sum_{1\leq t\leq T}\ell_t(y_t)
    &\stackrel{\text{\ref{it:divisible-segment-regret-proof-proof-0}}}\leq \sum_{1\leq t\leq T} \ell_t(g(\vx_t)) + (A + B)\cdot(1 + \log T)
    \\
    &\stackrel{\text{\ref{it:divisible-segment-regret-proof-proof-1}}}\leq \sum_{1\leq t\leq T} \ell_t(f(\vx_t)) + (A + B)\cdot(1 + \log T)
    \text,
\)
where we
\begin{enumerate*}[label=(\alph*)]
    \item
    \label{it:divisible-segment-regret-proof-proof-0}
    plugged \ref{eq:divisible-segment-regret-proof-1} into \eqref{eq:divisible-segment-regret-proof-0} for $z_t=u_t$ and
    \item 
    \label{it:divisible-segment-regret-proof-proof-1}
    choose $f\in\sF$, where $f(\vx, y) = (1-v)\cdot\vw^\T\vx + v\cdot y$ and $v=0$.    
\end{enumerate*}

\medskip\par\noindent\textit{Lemma~\ref{lem:divisible-segment-regret}\ref{lem:divisible-segment-regret-b}.}
We choose $f(\vx, y) = (1-v)\cdot\vw^\T\vx + v \cdot y$, where $\vw = \vec0$ and $v=1$ and plug this into \eqref{eq:divisible-segment-regret-proof-0} for $z_t = v_t$, so
\(
    \sum_{1\leq t\leq T} \ell(y_t) \leq \sum_{1\leq t\leq T} \ell_t(f(\vx_t, v_t)) + B\cdot(1 + \log T)
    \text.
    \tag*{\qedhere}
\)
\end{proof}

\subsection{Analysis}
We can now combine our main result on the HPF meta algorithm with our particular choice of learners for linear functions from the previous sections to obtain a loss guarantee.

\begin{thm}[LHPF Loss]
\label{thm:lhpf-regret}
Consider Algorithm~\ref{fig:hpf} with
\begin{itemize}
    \item
    learner $L_0$ with initial parameters (state and forecaster, uniform for all $\sS\in\sD$) as specified in Algorithm~\ref{fig:divisible-segment-learner},
    \item
    learner $L_1$ with initial parameters (state and forecaster, uniform for all $\sS\in\sH\setminus\sD$)) as specified in Algorithm~\ref{fig:indivisible-segment-learner} and
    \item
    let the regularity conditions in Lemma~\ref{lem:indivisible-segment-regret} hold.
\end{itemize}
For any CPF with partition \sP\ induced by hierarchical partition \sH\ and forecasters $\{f_\sS : f_\sS(\vx) = \vw_\sS^\T\vx\text{ and }\vw_\sS\in\sW\}_{\sS\in\sP}$ LHPF satisfies
\(
    \loss^{\mathrm{LHPF}}_T
    \leq
    \loss^{\mathrm{CPF}}_T + 
    \left(A\cdot\sz\sP + B\cdot C_{\sP,\sH}\right)\cdot\left(1 + \log T\right)
    \text,
\)
where $C_{\sH,\sP} := \left\lvert\left\{\sS\in\sH : \sS \text{ is divisible there ex. }\sS'\in\sP\text{ \st\ }\sS\supseteq\sS'\right\}\right\rvert$.
\end{thm}
\begin{proof}
We get 
\(
    &\loss^{\mathrm{LHPF}}_T - \loss^{\mathrm{CPF}}_T
    \\
    &\qquad\stackrel{\text{\ref{it:lhpf-regret-proof-0}}}\leq 
    \sum_{\mathclap{\substack{\sS\in\sP: \\\text{$\sS$ indiv.}}}} A (1 + \log T) + 
    \sum_{\mathclap{\substack{\sS\in\sP: \\\text{$\sS$ div.}}}} (A + B) (1 + \log T) +
    \sum_{\mathclap{\substack{\sS\in\sH: \\\sS\supset\sS'\in\sP\\}}} B (1 + \log T)
    \\
    &\qquad\stackrel{\text{\ref{it:lhpf-regret-proof-1}}}=
    \bigg(
    A\cdot\sz\sP + \sum_{\mathclap{\substack{\sS\in\sP: \\\text{$\sS$ div.}}}} B + \sum_{\mathclap{\substack{\sS\in\sH:\\\sS\supset\sS'\in\sP}}} B
    \bigg) \cdot (1 + \log T)
    \\
    &\qquad\stackrel{\text{\ref{it:lhpf-regret-proof-2}}}=
    (A\cdot\sz\sP + B\cdot C_{\sH,\sP})\cdot (1 + \log T)
    \text,
\)
where we 
\begin{enumerate*}[label=(\alph*)]
    \item
    \label{it:lhpf-regret-proof-0}
    applied lemmas \ref{lem:indivisible-segment-regret} and \ref{lem:divisible-segment-regret},
    \item
    \label{it:lhpf-regret-proof-1}
    merged the sum over $A$'s and rearranged,
    \item
    \label{it:lhpf-regret-proof-2}
    transformed the sum ranges,
\end{enumerate*}
\(
    &\{\sS\in\sP:\sS\text{ is divisible}\}
    \tag*{(range of left sum)}
    \\
    &\qquad=
    \{\sS\in\sH:\sS\text{ is divisible and there exists }\sS'\in\sP\text{ \st\ }\sS = \sS'\}
    \text,
    \\
    &\{\sS\in\sH:\sS\supset\sS'\text{ for some }\sS'\in\sP\}
    \tag*{(range of right sum)}
    \\
    &\qquad=
    \{\sS\in\sH:\sS\text{ is divisible and there exists }\sS'\in\sP\text{ \st\ }\sS\supset\sS'\}
    \text,    
\)
merged the sums and rearranged.
\end{proof}

\noindent
Note that the loss bound in Theorem~\ref{thm:lhpf-regret} unveils an interesting structure of the regret of LHPF suffers relative to a CPF competitor.

First, there is regret for learning the forecasters of the competing CPF:
For every forecaster $f_\sS$ associated to a segment \sS\ from the CPF's partition \sP\ we pay regret at most $A\cdot(1 + \log T)$ to learn (the parameters $\vw_\sS$ of) $f_\sS$.
This regret is induced by Algorithm~\ref{fig:indivisible-segment-learner} (and its embedded version in Algorithm~\ref{fig:divisible-segment-learner}).

Second, there is regret for learning the partition of the competing CPF.
We have to pay at most regret $B\cdot(1 + \log T)$ to learn the partition \sP, more precisely how it is embedded in the hierarchical partition \sH.
To understand this consider the following recursive procedure:
Consider (sub-)segment \sS\ of feature space, initially $\sS=\sX$.
For $\sS\in\sP$, we stop the recursion and pay regret at most $B\cdot(1 + \log T)$, if $\sS$ is divisible or pay regret $0$, if \sS\ is indivisible (we only need to encode ``end-of-recursion'' where further recursion is possible, that is for divisible segments);
for $\sS\notin\sP$, recurse on segments $\sS_1, \sS_2, \dots$ that divide \sS.
This regret is introduced by Switching embedded in Algorithm~\ref{fig:divisible-segment-learner}.

\subsection{Choosing the Hierarchical Partition \sH}

\paragraph{Overview.}
Since \sH\ determines how effective HPF can exploit the feature space structure to generate specialized predictions its choice has significant impact on HPF's predictive power.
In the following we give three examples.

\paragraph{Fixed and Domain-specific.}
If the domain is well-structured it often is easy to exploit domain-specific information.
To illustrate this consider a forecasting problem over a bounded 2-dimensional grid (as we did for our experiments, see Section~\ref{sec:experiments}) where the relation between features and targets varies smoothly depending on grid location.
So we may assume a feature vector $\vx = (x_1', x_2', \dots, x_n', u, v)^\T$ has components $u$ and $v$ that encode a grid position (\ie\ ``side-information'') and $\vx' = (x_1', x_2' \dots, x_n')^\T$ is other information used for forecasting.
Now a reasonable choice of \sH\ can be a quad-tree decomposition of the grid at a certain tree depth.

\paragraph{Fixed and Domain-agnostic.}
In case there is no further information a randomized hierarchical partition based on half-spaces has proven to be useful \cite{veness2020gln}.
We construct \sH\ as follows:
Consider a segment \sS, initially $\sS = \sX\subseteq\RR^n$.
Now draw vector $\vec a$ normally with mean $\vec0$ and unit variance $\vec I$ and draw $b$ normally with mean $\mu$ and variance $\sigma$ (note these are hyperparameters and may depend on \sS).
The hyperplane $\vx \mapsto \frac{\vec a^\T}{\lVert\vec a\rVert} \vx - b$ (note the normal is isotropic) divides \sS\ into halfspaces $\sS_1$ and $\sS_2$ which we add to the (intermediate) hierarchical partition.
We now repeat this procedure recursively on $\sS_1$ and $\sS_2$ or stop if a stopping criterion is met, for instance if the number of recursion steps exceeds a threshold.

\paragraph{Adaptive and Domain-agnostic.}
If the feature space allows a distance metric $d$, then it is possibly to efficiently maintain a hierarchical partition through Cover Trees (CTs) \cite{beygelzimer06}.
Consider a CT with depth $N$ and some parameter $\delta > 1$.
The nodes of the CT correspond to data points (\eg\ formed a growing set of feature vectors) \st\ siblings $u$ and $v$ with depth $n$ are far away in the sense that $d(u, v) > \delta^{N - n}$ and any child $c$ of $u$ is close to $u$ in the sense $d(u, c) \leq \delta^ {N-n}$.
As noted in \cite{tziortziotis14} this tree structure induces a context tree similar to \cite{willems1995ctw} on the data points, in other words a hierarchic partition.
For more details we defer the reader to \cite{tziortziotis14}.

\section{Experiments}
\label{sec:experiments}

\subsection{Specification}

\paragraph{The Setting.}
We consider forecasting radar-based precipitation measurements over the UK\footnote{%
The data set was sourced from \href{https://www.metoffice.gov.uk/research/climate/maps-and-data/data/haduk-grid/datasets}{data collected from the UK Met Office NIMROD system} licensed by the UK Met Office under the \href{https://www.nationalarchives.gov.uk/doc/open-government-licence/version/3/}{Open Government Licence 3.0}.
} as previously done in \cite{ravuri2021gan}.
% In Figure~TODO we show a typical example of the data, a composition of (moving) precipitation structures.
The precipitation data is represented as a 1536 x 1280 grayscale video (lossless compression):
every pixel corresponds to the average precipitation value in mm/hr measured over a 1 km x 1 km area quantised to 12 bits,
the data ranges from 1 Jan 2016 to 31 Dec 2019 sampled in 5 min intervals.
Thus at time $t$, for any fixed location $\vu = (x, y)$ (the 1 km² grid cell), our goal is to predict the precipitation at time $t + H$  for $H= 5\text{ min}, 10\text{ min}, \dots$.
The dataset split matches \cite{ravuri2021gan}:
we use 2017 and 2018 for hyperparameter selection, 2019 is test data.
We constrain our evaluation to those pixels that have radar coverage over the area spanned by a 100 pixel radius to ensure sufficient context for forecasting.

\paragraph{Motion Estimation.}
Besides the pixel data we found it useful to incorporate motion information by estimating the motion field implied by the moving precipitation structures.
To estimate the motion vector at location \vu\ we employ Switching over a set of motion vector candidates $\vec d_1, \vec d_2, \dots$.
Each candidate $\vec d$ corresponds to an expert that predicts the pixels surrounding \vu\ by the pixels surrounding $\vu -\vec d$.
Hence, we use the squared error implied by pixel matching as loss and constrain pixel matching to the circular patch with radius 33.
Finally, we use the motion vector candidate (expert) with maximum Switching weight as motion vector estimate.
% This avoids rounding, as motion vectors must be over integers.
We use the union of $4r$ vectors spaced uniformly on the outline of a circle with radius $r=1, 2, 4, 8$.
To smooth the motion field we estimate motion vectors for every 8-th pixel in x- and y-direction and use linear interpolation to fill the gaps.

\paragraph{Features.}
Note that every combination of location $\vu$ and forecasting horizon $H$ likely has different feature vectors.
Our feature vectors are based on a set of context pixels determined by motion estimation.
At time $t$ a pixel at location $\vu$ has an associated motion vector estimate $\vec d_{\vec u}$ to track where the pixel was located at time $t-1$ (``where it came from''): $\vu - \vec d_{\vec u}$.
%This implicitly assumes the motion vector estimates don't vary from time $t-1$ and time $t$.
Now we construct the feature vector $\vx$ to predict a pixel at location $\vu$ at time $t+H$ as follows:
First, by accumulating the motion vector estimates we determine the path $\vu_0, \vu_1, \dots, \vu_H$ of this pixel over $H$ time steps, that is: $\vu_0 = \vu, \vu_1 = \vu_0 - \vec d_{\vu_0}, \vu_2 = \vu_1 - \vec d_{\vu_1}, \dots$.
Intuitively this means at time $t$ the pixel at location $\vu_H$ will be at location $\vu$ at time $t+H$.
So to predict the pixel at location $\vu$ at time $t+H$ we use the set of context pixels close to $\vu_H$ at time $t$.
We construct $\vx$ by rotating the circular patch of all pixels with distance at most 7 from $\vu_H$ by $-\angle(\vu_H, \vu)$ to make the context pixels invariant to the orientation of $\vu_H$ relative to $\vu$.

\paragraph{Forecaster.}
Our main forecaster is based on LHPF. Algorithm~\ref{fig:hpf} with Algorithm~\ref{fig:indivisible-segment-learner} as learner for indivisible segments and Algorithm~\ref{fig:divisible-segment-learner} as learner for divisible segments, both using the squared error.
We use a quad-tree partitioning of the image coordinates with 6 levels as hierarchical partition.

\paragraph{Baselines and Metrics.}
We compare our forecasts to the naive persistence forecaster, to PySTEPS, a simulation-based forecaster, to the GAN from \cite{ravuri2021gan} and UNet, two deep learning baselines.
For details on the baseline models we defer the reader to \cite{ravuri2021gan}.
To measure the forecasting performance we consider the MSE and the Critical Success Index (CSI) at precipitation thresholds 1 mm/hr, 2 mm/hr, 4 mm/hr and 8 mm/hr, see \cite{ravuri2021gan} for details.
The CSI measures how good we estimate the location of precipitation at a given intensity, it is of domain-specific interest;
he MSE is of interest, since LHPF attempts to asymptotically minimize the MSE.

\subsection{Results}

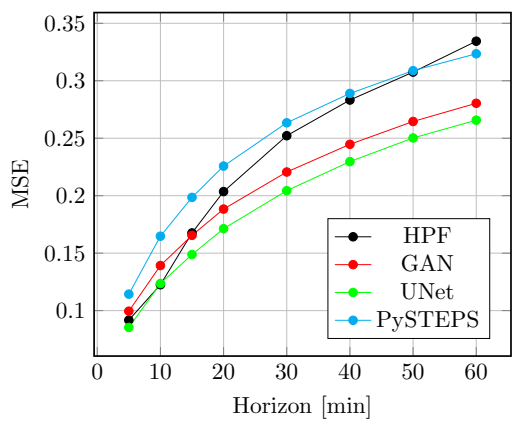
\begin{figure}
    \centering
    \begin{tikzpicture}[scale=0.8]
        \begin{axis}[axis line style=thick,xmajorgrids,ymajorgrids,legend style={anchor={south east},at={(0.95,0.05)}},xlabel={Horizon [min]},ylabel={MSE}]
        \pgfplotstableread[col sep=comma]{mse.csv}\mse;
        \addplot[color=black, mark=*, mark options={solid}] table [x=Horizon, y=HPF] {\mse};
        \addlegendentry{HPF}
        \addplot[color=red, mark=*, mark options={solid}] table [x=Horizon, y=GAN] {\mse};
        \addlegendentry{GAN}
        \addplot[color=green, mark=*, mark options={solid}] table [x=Horizon, y=UNet] {\mse};
        \addlegendentry{UNet}
        \addplot[color=cyan, mark=*, mark options={solid}] table [x=Horizon, y=PySTEPS] {\mse};
        \addlegendentry{PySTEPS}
        %\addplot table [x=Horizon, y=Persistence] {\mse};
        \end{axis}
    \end{tikzpicture}    
    \caption{MSE breakdown over forecasting horizon for various forecasters on the test period (2019).}
    \label{fig:experiment-mse}
\end{figure}

\begin{figure}
    \centering
    \begin{tikzpicture}[scale=0.8]
        \begin{groupplot}[group style={group size=2 by 2,xticklabels at=edge bottom,yticklabels at=edge left,xlabels at=edge bottom,ylabels at=edge left,horizontal sep=3em,vertical sep=4em},axis line style=thick,xmajorgrids,ymajorgrids,legend style={at={(0.95,0.95)}},xlabel={Horizon [min]},ylabel={CSI},xmin=0,xmax=65,ymin=-0.1,ymax=0.7]
        \nextgroupplot[title={\bfseries 1 mm/hr}]
        \pgfplotstableread[col sep=comma]{csi1.csv}\mse;
        \addplot[color=black, mark=*, mark options={solid}] table [x=Horizon, y=HPF] {\mse};
        \addplot[color=red, mark=*, mark options={solid}] table [x=Horizon, y=GAN] {\mse};
        \addplot[color=green, mark=*, mark options={solid}] table [x=Horizon, y=UNet] {\mse};
        \addplot[color=cyan, mark=*, mark options={solid}] table [x=Horizon, y=PySTEPS] {\mse};
        \nextgroupplot[title={\bfseries 2 mm/hr}]
        \pgfplotstableread[col sep=comma]{csi2.csv}\mse;
        \addplot[color=black, mark=*, mark options={solid}] table [x=Horizon, y=HPF] {\mse};
        \addlegendentry{HPF}
        \addplot[color=red, mark=*, mark options={solid}] table [x=Horizon, y=GAN] {\mse};
        \addlegendentry{GAN}
        \addplot[color=green, mark=*, mark options={solid}] table [x=Horizon, y=UNet] {\mse};
        \addlegendentry{UNet}
        \addplot[color=cyan, mark=*, mark options={solid}] table [x=Horizon, y=PySTEPS] {\mse};
        \addlegendentry{PySTEPS}
        \nextgroupplot[title={\bfseries 4 mm/hr}]
        \pgfplotstableread[col sep=comma]{csi4.csv}\mse;
        \addplot[color=black, mark=*, mark options={solid}] table [x=Horizon, y=HPF] {\mse};
        \addplot[color=red, mark=*, mark options={solid}] table [x=Horizon, y=GAN] {\mse};
        \addplot[color=green, mark=*, mark options={solid}] table [x=Horizon, y=UNet] {\mse};
        \addplot[color=cyan, mark=*, mark options={solid}] table [x=Horizon, y=PySTEPS] {\mse};
        \nextgroupplot[title={\bfseries 8 mm/hr}]
        \pgfplotstableread[col sep=comma]{csi8.csv}\mse;
        \addplot[color=black, mark=*, mark options={solid}] table [x=Horizon, y=HPF] {\mse};
        \addplot[color=red, mark=*, mark options={solid}] table [x=Horizon, y=GAN] {\mse};
        \addplot[color=green, mark=*, mark options={solid}] table [x=Horizon, y=UNet] {\mse};
        \addplot[color=cyan, mark=*, mark options={solid}] table [x=Horizon, y=PySTEPS] {\mse};
        \end{groupplot}
    \end{tikzpicture}    
    \caption{CSI breakdown over forecasting horizon for various forecasters on the test period (2019).}
    \label{fig:experiment-csi}
\end{figure}
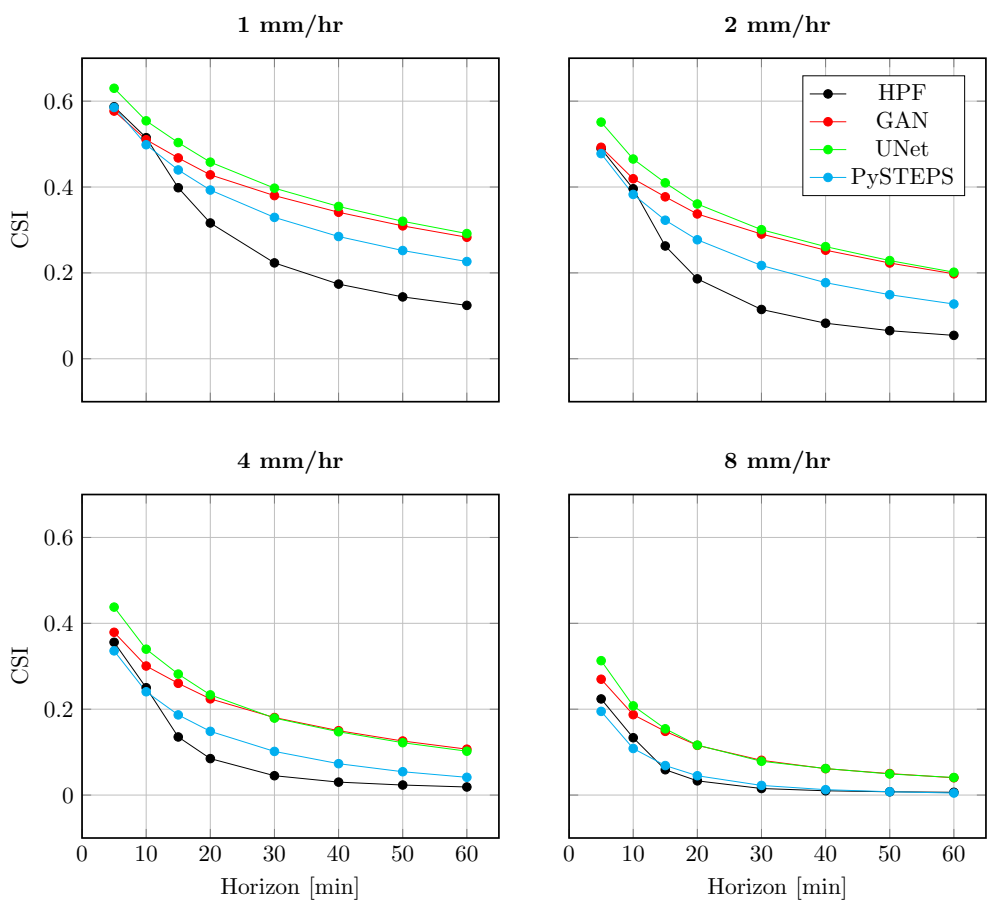

Figure~\ref{fig:experiment-mse} and Figure~\ref{fig:experiment-csi} we break down the MSE and CSI over the forecasting horizons.
Broadly speaking, depending on the exact evaluation setting LHPF is comparable to the deep learning baselines for short forecasting horizons up to roughly 15 (MSE) or 20 minutes (MSE) and starting from there performs degrades quickly and the other forecasters perform significantly better.
For the CSI this pattern is more pronounced than for the MSE.
This behavior is plausible:
Our LHPF implementation implicitly minimizes the MSE which in general seems to lead to blurry predictions.
Clearly, blurring reduces the spatial accuracy of predictions, hence the CSI - which measures spatial accuracy - will be negatively affected.
Also, the motion estimation procedure that LHPF incorporates to build features accumulates motion estimation errors with increasing forecasting horizon.
This will have a negative impact on the forecasting performance at long horizons.

Besides the deficits on longer-term horizons the short-horizon performance is remarkable given that LHPF has orders of magnitude less parameters than the best performing deep learning models GAN and UNet, learns and predicts on the fly and based on that executes significantly faster than the deep learning models.

\section{Conclusion}
\label{sec:conclusion}

In this work we introduce the family of Hierarchical Partition Forecasters (HPFs) that follow an intuitive divide-and-conquer principle:
Divide the feature space hierarchically, assign specialized forecasters to the evolving feature sub-spaces and blend their forecasts to obtain a good forecast. Blending compensates for not knowing which feature sub-spaces are worth specializing.
In this light an idealized forecaster \emph{may know} how to partition the feature space into sub-spaced and \emph{may know} the parameters of a specialized forecaster for either part of the partition.
We term such an idealized forecaster Constant Partition Forecaster (CPF).

We specify an online meta-algorithm that estimates a sequence of HPFs and regards the learning algorithms for learning the specializing forecasters within HPF (``learners'') as parameters.
This meta-algorithm can perform almost as well as the best CPF in hindsight, given the learners are powerful enough.
Our analysis of this meta-algorithm confirms an intuitive view:
First, the regret incurred relative to some CPF consists of learning the CPF partitioning and the predictors associated to every partition;
Second, the regret guarantees on the actual learning algorithms determine the regret guarantees of the meta-algorithm.
Besides these very abstract results we consider a concrete example:
learning LHPFs - HPF with linear forecasters - with online learning algorithms based on online second order optimization and Switching.
Our results reveal that for exp-concave losses the proposed approach yields $O(\log T)$ regret relative to a CPF with linear forecasters.
An experimental study underpins the usefulness of our approach, as we achieve performance comparable to significantly more complex deep learning models in various settings, yet it also reveals limitations caused my model complexity.

There are several research directions for future work.
Our analysis framework links theoretic guarantees on learners to those of the entire meta-algorithm.
Hence, exploring learners with stronger theoretic guarantees defines an interesting research direction, as these guarantees would carry forward to the meta-algorithm.
For instance, learners with shifting regret guarantees should imply that shifting regret guarantees (of the meta-algorithm) \wrt\ shifting CPFs rather than a fixed CPF.
%Also, adopting edge weighting TODO would allow for competing with a richer class of hierarchical partitions.
Finally, the existing algorithms and, more generally, algorithmic extensions deserve further experimental analysis.

\paragraph{Acknowledgement.}
For helpful comments, discussions and engineering support the author would like to thank, in alphabetic order,
Elliot Catt,
Chris Dyer,
Tim Genewein,
George Holland,
Marcus Hutter,
Remi Lam,
Shakir Mohammed,
Suman Ravuri,
Alvaro San\-chez-Gon\-zalez,
Jacklynn Stott,
Joel Veness
and
Matthew Willson.

\bibliographystyle{plain}
\bibliography{references.bib}

\begin{thebibliography}{10}

\bibitem{beygelzimer06}
Alina Beygelzimer, Sham~M. Kakade, and John Langford.
\newblock Cover trees for nearest neighbor.
\newblock In William~W. Cohen and Andrew~W. Moore, editors, {\em Machine
  Learning, Proceedings of the Twenty-Third International Conference {(ICML}
  2006), Pittsburgh, Pennsylvania, USA, June 25-29, 2006}, volume 148 of {\em
  {ACM} International Conference Proceeding Series}, pages 97--104. {ACM},
  2006.

\bibitem{veness2020ggln}
David Budden, Adam~H. Marblestone, Eren Sezener, Tor Lattimore, Gregory Wayne,
  and Joel Veness.
\newblock Gaussian gated linear networks.
\newblock In {\em Advances in Neural Information Processing Systems 33: Annual
  Conference on Neural Information Processing Systems 2020, NeurIPS 2020,
  December 6-12, 2020, virtual}, 2020.

\bibitem{hazan2006ons}
Elad Hazan, Adam Kalai, Satyen Kale, and Amit Agarwal.
\newblock Logarithmic regret algorithms for online convex optimization.
\newblock In {\em Learning Theory, 19th Annual Conference on Learning Theory,
  {COLT} 2006, Pittsburgh, PA, USA, June 22-25, 2006, Proceedings}, volume 4005
  of {\em Lecture Notes in Computer Science}, pages 499--513. Springer, 2006.

\bibitem{herbster1995fixedshare}
Mark Herbster and Manfred~K. Warmuth.
\newblock Tracking the best expert.
\newblock In {\em Machine Learning, Proceedings of the Twelfth International
  Conference on Machine Learning, Tahoe City, California, USA, July 9-12,
  1995}, pages 286--294. Morgan Kaufmann, 1995.

\bibitem{ravuri2021gan}
Suman~V. Ravuri, Karel Lenc, Matthew Willson, Dmitry Kangin, R{\'{e}}mi Lam,
  Piotr Mirowski, Megan Fitzsimons, Maria Athanassiadou, Sheleem Kashem, Sam
  Madge, Rachel Prudden, Amol Mandhane, Aidan Clark, Andrew Brock, Karen
  Simonyan, Raia Hadsell, Niall~H. Robinson, Ellen Clancy, Alberto Arribas, and
  Shakir Mohamed.
\newblock Skilful precipitation nowcasting using deep generative models of
  radar.
\newblock {\em Nat.}, 597(7878):672--677, 2021.

\bibitem{tziortziotis14}
Nikolaos Tziortziotis, Christos Dimitrakakis, and Konstantinos Blekas.
\newblock Cover tree bayesian reinforcement learning.
\newblock {\em J. Mach. Learn. Res.}, 15(1):2313--2335, 2014.

\bibitem{erven2007switching}
Tim van Erven, Peter Grunwald, and Steven de~Rooij.
\newblock Catching up faster in bayesian model selection and model averaging.
\newblock In {\em Advances in Neural Information Processing Systems 20,
  Proceedings of the Twenty-First Annual Conference on Neural Information
  Processing Systems, Vancouver, British Columbia, Canada, December 3-6, 2007},
  pages 417--424. Curran Associates, Inc., 2007.

\bibitem{veness2020gln}
Joel Veness, Tor Lattimore, David Budden, Avishkar Bhoopchand, Christopher
  Mattern, Agnieszka Grabska{-}Barwinska, Eren Sezener, Jianan Wang, Peter
  Toth, Simon Schmitt, and Marcus Hutter.
\newblock Gated linear networks.
\newblock In {\em Thirty-Fifth {AAAI} Conference on Artificial Intelligence,
  {AAAI} 2021, Thirty-Third Conference on Innovative Applications of Artificial
  Intelligence, {IAAI} 2021, The Eleventh Symposium on Educational Advances in
  Artificial Intelligence, {EAAI} 2021, Virtual Event, February 2-9, 2021},
  pages 10015--10023. {AAAI} Press, 2021.

\bibitem{veness2012cts}
Joel Veness, Kee~Siong Ng, Marcus Hutter, and Michael~H. Bowling.
\newblock Context tree switching.
\newblock In {\em 2012 Data Compression Conference, Snowbird, UT, USA, April
  10-12, 2012}, pages 327--336. {IEEE} Computer Society, 2012.

\bibitem{willems1998ctw}
Frans M.~J. Willems.
\newblock The context-tree weighting method : Extensions.
\newblock {\em {IEEE} Trans. Inf. Theory}, 44(2):792--798, 1998.

\bibitem{willems1995ctw}
Frans M.~J. Willems, Yuri~M. Shtarkov, and Tjalling~J. Tjalkens.
\newblock The context-tree weighting method: basic properties.
\newblock {\em {IEEE} Trans. Inf. Theory}, 41(3):653--664, 1995.

\bibitem{zinkevich2003ogd}
Martin Zinkevich.
\newblock Online convex programming and generalized infinitesimal gradient
  ascent.
\newblock In {\em Machine Learning, Proceedings of the Twentieth International
  Conference {(ICML} 2003), August 21-24, 2003, Washington, DC, {USA}}, pages
  928--936. {AAAI} Press, 2003.

\end{thebibliography}

\clearpage
\appendix
\section{Switching}

\subsection{Introduction}

\paragraph{Overview.}
Switching is method for prediction with expert advise.
Prediction with expert advise is a special case of the sequential prediction problem introduced earlier:
That is, in the $t$-th round $m$ expert predictions $\vy^1_t, \vy^2_t, \dots, \vy^m_t\in\sY$ make up the the feature vector, $\vx_t = (\vy^1_t, \vy^2_t, \dots, \vy^m_t)$.
The forecaster combines these expert predictions.
Algorithm~\ref{fig:switching} depicts Switching in the sequential prediction (with expert advise) setting.
We denote the loss accumulated by Algorithm~\ref{fig:switching} by
\(
    \loss^{SW}_T := \sum_{1\leq t\leq T} \ell_t(\vy_t)
    \text.
    \label{eq:switching-loss}
\)

\begin{alg}
\centering
\small
\fboxsep=0.4\baselineskip
\fbox{\begin{minipage}{0.9\textwidth}
    \noindent
    \begin{tabular}{@{}l@{}@{\hskip2pt}c@{\hskip2pt}p{0.8\columnwidth}@{}}
        \textbf{Input}&\textbf:&
        Sequence $\vx_{1:T}$ of predictions of $m$ experts,\newline
        sequence $\alpha_{1:T}$ of switching rates from $(0, 1)$ and\newline
        sequence $\ell_{1:T}$ of $\eta$-exp. concave loss functions.
        \\[0.2\baselineskip]
        \textbf{Output}&\textbf:& (Mixed) predictions $\vy_{1:T}$.
        \\[0.2\baselineskip]
    \end{tabular}
    \\[.5\baselineskip]
    Set $\beta^i_0 = \frac1m$, for all $1\leq i\leq m$.\\[0.5\baselineskip]
    For $t = 1, 2, \dots, T$ do:
    \begin{enumerate}[label=\arabic*.,ref=\arabic*]
        \item
        Observe expert predictions $\vx_t = (\vy_t^1, \vy_t^2, \dots, \vy_t^m)$ and combine,
        \(
            \vy_t
            = \frac{\sum_{1\leq i\leq m} \beta^i_{t-1}\cdot \vy^i_t}{\sum_{1\leq i\leq m} \beta^i_{t-1}}
            \text.
        \)
        \item
        Observe loss $\ell_t$ and update weights
        \(
            \beta^i_t
            = (1-\alpha_t)\beta^i_{t-1}\cdot e^{-\eta\ell^i_t} + \frac{\alpha_t}{m-1}\sum_{j\neq i}\beta^j_{t-1}\cdot e^{-\eta\ell^j_t}
            \text,
            \label{eq:switching-update}
        \)
        where the $i$-th expert's loss  is $\ell^i_t := \ell_t(\vy^i_t)$, for $1\leq i\leq m$.
    \end{enumerate}
\end{minipage}}
\caption{Switching.}
\label{fig:switching}
\end{alg}

\paragraph{Switching Sequences, their Loss and Prior.}
In the following we use a sequence $i_{1:T}$ over $\{1, 2, \dots, m\}$ to formalize an \emph{idealized} forecaster for prediction with expert advise.
In the $t$-th round this forecaster simply predicts the $i_t$-th expert's prediction $\vy^{i_t}_t$ and accumulates loss
\(
    \loss(i_{1:T}) := \sum_{1\leq t\leq T} \ell_t^{i_t}
    \text{,\enskip where\enskip }
    \ell_t^i := \ell_t(\vy^i_t)
    \text.
\)
For brevity we refer to a specific idealized forecaster as \emph{switching sequence} $i_{1:T}$.
The theoretical guarantees on Switching are based on a link between the weights $\beta^i_t$ to a prior over switching sequences.
Given a sequence $\alpha_{1:\infty}$ of switching rates over $(0, 1)$ this \emph{switching prior} as given by
\(
    w(i_{1:t}) :=
    \begin{cases}
    1, & \text{if $t=0$},\\
    \frac1m\text, & \text{if $t=1$,}\\
    w(i_{<t})\cdot(1-\alpha_{t-1})\text, & \text{if $t>1$ and $i_t=i_{t-1}$,}\\
    w(i_{<t})\cdot\frac{\alpha_{t-1}}{m-1}\text, & \text{if $t>1$ and $i_t\neq i_{t-1}$.}
    \end{cases}
    \text.
    \label{eq:switching-prior}
\)

\subsection{Analysis}
\paragraph{Technical Lemmas.}
We now parenthesize some technical lemmas before proving a loss bound for Switching.
As mentioned before the main technical point is to establish a link between the weights $\beta^j_t$ and the switching prior $w$ by recognizing that $\beta^j_t$ is the expected exponentiated loss $e^{-\eta \loss(i_{1:t})}$ of all switching sequences with length $t+1$ ending with $i_{t+1}=j$.

\begin{lem}[$\beta$'s are Expectations]
\label{lem:switching-prior-sum}
For $t\geq0$ we have
\(
    \beta^j_t = \sum_{\substack{i_{1:t+1}:\\i_{t+1}=j}} w(i_{1:t+1})\cdot e^{-\eta \loss(i_{1:t})}
    \text.
\)
\end{lem}
\begin{proof}
Our prove is by induction on $t$.

\medskip\par\noindent\textit{Base: $t=0$.} ---
We have $\beta_0^j = w(j)\cdot e^{-\eta \loss(i_{1:0})} = \frac1m$, since $w(j)=\frac1m$ and $\loss(i_{1:0})=0$.

\medskip\par\noindent\textit{Step: $t>0$.} --- 
We have
\(
    \beta^j_t
    &\stackrel{\text{\ref{it:switching-prior-sum-proof-0}}}= (1-\alpha_t) \cdot \sum_{\substack{i_{1:t}:\\i_t=j}} w(i_{1:t})\cdot e^{-\eta \loss(i_{<t})}\cdot e^{-\eta\ell^j_t}\\
    &\qquad\qquad\qquad + \frac{\alpha_t}{m-1}\cdot\sum_{u\neq j} \sum_{\substack{i_{1:t}\\i_t=u}} w(i_{1:t})\cdot e^{-\eta \loss(i_{<t})}\cdot e^{-\eta \ell^u_t}
    \\
    &\stackrel{\text{\ref{it:switching-prior-sum-proof-1}}}= \sum_{\substack{i_{1:t}:\\i_t=j}} w(i_{1:t})(1-\alpha_t)\cdot e^{-\eta \loss(i_{1:t})} + \sum_{\substack{i_{1:t}\\i_t\neq j}} w(i_{1:t})\frac{\alpha_t}{m-1}\cdot e^{-\eta \loss(i_{1:t})}
    \\
    &\stackrel{\text{\ref{it:switching-prior-sum-proof-2}}}= \sum_{\substack{i_{1:t}:\\i_t=j}} w(i_{1:t}j)\cdot e^{-\eta \loss(i_{1:t})} + \sum_{\substack{i_{1:t}:\\i_t\neq j}} w(i_{1:t}j)\cdot e^{-\eta \loss(i_{1:t})}
    \\    
    &\stackrel{\text{\ref{it:switching-prior-sum-proof-3}}}=\sum_{\substack{i_{1:t+1}:\\i_{t+1}=j}} w(i_{i:t+1})\cdot e^{-\eta \loss(i_{1:t})}
    \text,
\)
where we 
\begin{enumerate*}[label=(\alph*)]
    \item\label{it:switching-prior-sum-proof-0}
    plugged the induction hypothesis for $t-1$ into \eqref{eq:switching-update},
    \item\label{it:switching-prior-sum-proof-1}
    used the definition of $\loss(i_{1:t})$ and rearranged,
    \item\label{it:switching-prior-sum-proof-2}
    used \eqref{eq:switching-prior} and
    \item\label{it:switching-prior-sum-proof-3}
    finally rearranged to conclude the proof.
\end{enumerate*}
\end{proof}

\noindent
Next, we introduce two invariants on $\beta$-weighted losses.

\begin{lem}[Total Weight Invariant]
\label{lem:switching-beta-sum}
For $t>0$ we have $\sum_{1\leq j\leq m} \beta^j_{t-1}\cdot e^{-\eta \ell^j_t} = \sum_{1\leq j\leq m}\beta^j_t$.
\end{lem}
\begin{proof}
We get
\(
    \sum_{1\leq j\leq m}\beta^j_t
    &
    \stackrel{\text{\ref{it:switching-beta-sum-0}}}= (1-\alpha_t)\cdot\sum_{1\leq j\leq m} \beta^j_{t-1}\cdot e^{-\eta\ell^j_t} + \frac{\alpha_t}{m-1}\sum_{1\leq j\leq m} \sum_{\substack{1\leq u\leq m:\\u\neq j}}\beta^u_{t-1} \cdot e^{-\eta\ell^u_t}
    \\
    &
    \stackrel{\text{\ref{it:switching-beta-sum-1}}}= (1-\alpha_t)\cdot\sum_{1\leq j\leq m} \beta^j_{t-1}\cdot e^{-\eta\ell^j_t} + \frac{\alpha_t}{m-1}\cdot\sum_{1\leq u\leq m} (m-1)\cdot e^{-\eta\ell^u_t}
    \\
    &
    \stackrel{\text{\ref{it:switching-beta-sum-2}}}= \sum_{1\leq j\leq m}\beta^j_{t-1}e^{-\eta \ell^j_t}
    \text,
\)
where we
\begin{enumerate*}[label=(\alph*)]
    \item\label{it:switching-beta-sum-0}
    substituted \eqref{eq:switching-update},
    \item\label{it:switching-beta-sum-1}
    changed the order of summation
    $
        \sum_{1\leq j\leq m} \sum_{u\neq j}z_u
        =\allowbreak \sum_{1\leq u\leq m}\sum_{j\neq u}z_u
        =\allowbreak(m-1)\sum_{1\leq u\leq m} z_u
    $
    and finally
    \item\label{it:switching-beta-sum-2}
    used $0\leq \alpha_t\leq 1$ and rearranged.
\end{enumerate*}
\end{proof}

\begin{lem}[Monotonicity Invariant]
\label{lem:switching-potential}
The term $\sum_{1\leq j\leq m}\beta^j_t\cdot e^{\eta \loss^{SW}_t}$, is decreasing in $t$.
\end{lem}
\begin{proof}
For brevity let $L_t := \loss^{SW}_t$.
For $t>0$ we have
\(
    &e^{-\eta(L_t - L_{t-1})}
    \stackrel{\text{\ref{it:potential-proof-0}}}= e^{-\eta\ell_t(\vy_t)}
    \stackrel{\text{\ref{it:potential-proof-1}}}\geq \frac{\sum_{1\leq j \leq m}\beta^j_{t-1}\cdot e^{-\eta\ell^j_t}}{\sum_{1\leq j\leq m}\beta^j_{t-1}} 
    \stackrel{\text{\ref{it:potential-proof-2}}}=\frac{\sum_{1\leq j\leq m}\beta^j_t}{\sum_{1\leq j\leq m}\beta^j_{t-1}}
    \\
    \stackrel{\text{\ref{it:potential-proof-3}}}\iff&
    \sum_{1\leq j\leq m}\beta^j_{t-1}\cdot e^{\eta L_{t-1}}
    \geq
    \sum_{1\leq j\leq m}\beta^j_t\cdot e^{\eta L_t}
    \text,
\)
where we used
\begin{enumerate*}[label=(\alph*)]
    \item\label{it:potential-proof-0}
    the definition of $\loss^{SW}_t$,
    \item\label{it:potential-proof-1}
    that $\ell_t$ is $\eta$-exp-concave and the definition of $\ell^j_t$,
    \item\label{it:potential-proof-2}
    Lemma~\ref{lem:switching-beta-sum} and finally
    \item\label{it:potential-proof-3}
    rearranged.
\end{enumerate*}
\end{proof}

\paragraph{Loss Bound.}
We are now ready to state and prove that the loss Switching incurrs is not much worse that that of any (hence, the best in hindsight) switching sequence.

\begin{thm}[Switching Loss]
\label{thm:switching-regret-bound}
For any switching sequence $i_{1:T}$ over $\{1, 2, \dots, m\}$ we have
\(
    \loss_T^{SW}
    &\leq \loss(i_{1:T})
    \\
    &\quad\quad+ \frac1\eta\left[\log m + \sz\sT\log(m-1) + \sum_{t\in\sT}\log\frac1\alpha_t + \smash{\sum_{\mathclap{\substack{t\notin\sT,\\1\leq t< T}}}}\log\frac1{1-\alpha_t}\right]
    \text,
\)
where $\sT := \{1\leq t < T : i_t\neq i_{t+1}\}$.
(Note that sums range over $1\leq t<T$.)
\end{thm}
\begin{proof}
For brevity let $L_t := \loss^{\text{SW}}_t$.
We define the potential $\phi_t := \frac1\eta\log\sum_{1\leq j\leq m}\beta^j_t\cdot e^{\eta L_t}$ and obtain
\(
    0
    \stackrel{\text{\ref{it:switching-regret-bound-proof-0}}}= \phi_0
    \stackrel{\text{\ref{it:switching-regret-bound-proof-1}}}\geq \phi_T
    &= \frac1\eta\log\sum_{1\leq j\leq m}\beta^j_t\cdot e^{\eta L_T}
    \\
    &\stackrel{\text{\ref{it:switching-regret-bound-proof-2}}}=\frac1\eta\log\sum_{1\leq j\leq m} \sum_{\substack{i_{1:T+1}':\\i_{T+1}'=j}} w(i_{1:T+1}')\cdot e^{-\eta (\loss(i_{1:T}')- L_T)}
    \\
    &\stackrel{\text{\ref{it:switching-regret-bound-proof-3}}}\geq\frac1\eta\log\sum_{1\leq j\leq m} w(i_{1:T}j)\cdot e^{-\eta(\loss(i_{1:T})- L_T)}
    \\
    &\stackrel{\text{\ref{it:switching-regret-bound-proof-4}}}=\frac1\eta\log\left(w(i_{1:T})\cdot e^{-\eta(\loss(i_{1:T})-L_T)}\right)
    \\
    &\stackrel{\text{\ref{it:switching-regret-bound-proof-5}}}= L_T - \loss(i_{1:T}) - \frac1\eta\log\frac1{w(i_{1:T})}
    \text,
    \label{eq:switching-regret-bound-proof-0}
\)
where we used
\begin{enumerate*}[label=(\alph*)]
    \item\label{it:switching-regret-bound-proof-0}
    $\sum_{1\leq j\leq m}\beta^j_0 = 1$ and $L_0 = 0$,
    \item\label{it:switching-regret-bound-proof-1}
    by Lemma~\ref{lem:switching-potential} $\phi_t$ is decreasing in $t$,
    \item\label{it:switching-regret-bound-proof-2}
    Lemma~\ref{lem:switching-beta-sum}
    \item\label{it:switching-regret-bound-proof-3}
    dropping all terms $i_{1:T+1}' \neq i_{1:T}j$ in the innermost sum,
    \item\label{it:switching-regret-bound-proof-4}
    $w(i_{1:T}) =\allowbreak \sum_{1\leq j\leq m} w(i_{1:T}j)$ and 
    \item\label{it:switching-regret-bound-proof-5}
    rearranged.
\end{enumerate*}
From \eqref{eq:switching-prior} it is easy to see that
\(
    w(i_{1:T}) = \frac1m\frac1{(m-1)^{\sz\sT}} \cdot\prod_{t\in\sT}\alpha_t \cdot \prod_{\mathclap{\substack{t\notin\sT,\\1\leq t < T}}}(1-\alpha_t)
    \text,
\)
which we plug into \eqref{eq:switching-regret-bound-proof-0} and rearrange to end the proof.
\end{proof}

\noindent
Note that the regret term in Theorem~\ref{thm:switching-regret-bound} has a natural interpretation in the light of encoding the competing switching sequence $i_{1:T}$:
First, we pay $\log m$ bits to encode $i_1$.
Second, for either of the $\sz\sT$ switches we pay $\log(m-1)$ bits to encode switching from the current expert $j$ to one of the other $m-1$ experts from $\{1, 2, ...\dots m\}\setminus\{j\}$.
Finally, for positions $1 < t < T$ we encode whether a switch occurs from $t$ to $t+1$, paying $\log\frac1\alpha_t$ bits for a switch ($i_t\neq i_{t+1}$) and $\log\frac1{1-\alpha_t}$ for no switch ($i_t= i_{t+1}$).

\begin{cor}
\label{cor:switching-two-expert-regret-bound}
For a binary switching sequence $i_{1:T}$ with $n$ switches and switching rate $\alpha_t = (t+1)^{-1}$ we have
\(
    \loss^{\text{SW}}_T \leq \loss(i_{1:T}) + \frac1\eta \left[1 + (n+1) \log T\right]
    \text.
\)
\end{cor}
\begin{proof}
We combine Theorem~\ref{thm:switching-regret-bound} with $\log\frac1\alpha_t \leq \log T$ and $\sum_{t\notin\sT} \log\frac1{1-\alpha_t} \leq \sum_{1\leq t< T} \log\frac t{t-1} = \log T$.
\end{proof}

\end{document}